\documentclass{article}
\usepackage{microtype}
\usepackage{fullpage}
\usepackage{graphicx}
\usepackage{subcaption}
\usepackage[utf8]{inputenc} % allow utf-8 input
\usepackage[T1]{fontenc}    % use 8-bit T1 fonts
\usepackage{url}            % simple URL typesetting
\usepackage{booktabs}       % professional-quality tables
\usepackage{amsfonts}       % blackboard math symbols
\usepackage{nicefrac}       % compact symbols for 1/2, etc.
\usepackage{microtype}      % microtypography
\usepackage{multirow}
\usepackage{algorithm}
\usepackage{algorithmic}
\usepackage{adjustbox}
\usepackage{bbm}
\usepackage{bm}
\usepackage{mathrsfs}
\usepackage{amsmath} %数学公式
\usepackage[numbers]{natbib}
\usepackage{float}
\usepackage{diagbox}
\usepackage{amssymb}
\usepackage[dvipsnames]{xcolor}
\usepackage{enumitem}
\usepackage{tikz}
\usepackage{hyperref}
\usepackage{blkarray}
\usepackage{makecell}
\usepackage{footnote}

\newtheorem{thm}{Theorem}
\newtheorem{defn}{Definition}

\newtheorem{lemma}{Lemma}
\newtheorem{assume}{Assumption}
\newtheorem{remark}{Remark}

\newenvironment{proof}{Proof:}{\hfill$\square$}
\usepackage[toc,page]{appendix}
\newcommand{\PP}{\textit{Property P}}

\author{Dachao Lin\thanks{Academy for Advanced Interdisciplinary Studies,
		Peking University.
		\texttt{lindachao@pku.edu.cn}}
	\and
	Ruoyu Sun\thanks{Department of Industrial and Enterprise Engineering, Coordinate Science Lab (affiliated), University of Illinois Urbana{-}Champaign.
		\texttt{ruoyus@illinois.edu}} 
	\and
	Zhihua Zhang\thanks{School of Mathematical Sciences,
		Peking University.
		\texttt{zhzhang@math.pku.edu.cn}}
}
\title{On the Landscape of One-hidden-layer Sparse Networks and Beyond}

\begin{document}

\maketitle

\begin{abstract}
Sparse neural networks have received increasing interest due to their small size compared to dense networks.
Nevertheless, most existing works on neural network theory have focused on dense neural networks, and the understanding of sparse networks is very limited. 
In this paper we study the loss landscape of one-hidden-layer sparse networks.
First, we consider sparse networks with a dense final layer. 
We show that linear networks can have no spurious valleys under special sparse structures, and
non-linear networks could also admit no spurious valleys under a wide final layer. 
Second, we discover that spurious valleys and spurious minima can exist for wide sparse networks with a sparse final layer. 
This is different from wide dense networks which do not have spurious valleys under mild assumptions.  
\end{abstract}

\section{Introduction}
\paragraph{Motivation} Deep neural networks (DNNs) have achieved remarkable empirical successes in the fields of computer vision, speech recognition, and natural language processing, sparking great interests in the theory underlying their architectures and training. 
However, DNNs are typically highly over-parameterized, making them computationally expensive with large amounts of memory and computational power. 
For example, it may take more than 1 week to train ResNet-52 on ImageNet with a single GPU. 
Thus, DNNs are often unsuitable for smaller devices like embedded electronics, sensors and smartphones. 
And there is a pressing demand for techniques to optimize models with reduced model size, faster inference and lower power consumption.

Sparse networks, in which a large subset of the connections are non-existent, are a possible choice for small-size networks. 
It has been shown empirically that sparse networks can achieve performance comparable to dense networks \cite{han2015learning, gale2019state, louizos2017bayesian}. 
In recent years, many sparse networks are obtained from network pruning \cite{zhu2017prune, lee2018snip, liu2018rethinking, frankle2018lottery}.

There are many existing theoretical works on dense neural networks.
One popular theoretical line of research is to analyze the loss surface of neural networks \cite{sun2020global}. 
Although this line does not provide a full description of the algorithm behavior, it serves as a starting point for understanding DNNs.  
For deep linear networks, it is shown that every local minimum is a global minimum \cite{baldi1989neural, kawaguchi2016deep}.
For non-linear networks, researchers have found spurious local minima for various non-linear networks \cite{auer1996exponentially, zhou2017critical, swirszcz2016local, yun2018small, safran2018spurious}.
Despite the existence of suboptimal local minima, researchers discovered that non-linear neural nets can be trained to global minima.
One explanation is that such over-parameterized networks still exhibit nice geometrical landscape properties such as the absence of bad valleys and basins \cite{freeman2019topology, venturi2019spurious, li2018benefit}. 
\citet{nguyen2018loss} also considered sparse networks, showing no spurious ``valleys'' \footnote{There is a slight difference in the definition of spurious valleys in \citet{nguyen2018loss} with common adopted concept from \citet{freeman2019topology, venturi2019spurious}. 
We employ the later one in our paper.} under strictly increasing activations as long as the bias term is not pruned and the width (or the connections) of the final hidden layer is larger than the number of training samples. 
It is not clear whether a more general sparse network (e.g., allowing the final hidden layer to be sparse with few connections) still exhibits the same nice landscape property. 

\paragraph{Main contribution}
Since sparse networks have gained a great success in practice, we may wonder what kind of a sparse network is good for training, or still preserves the good landscape as dense networks?
In this work, we show that some specific sparsity could still provide a benign landscape. But generally, sparsity can deteriorate the loss landscape and create spurious minima and spurious valleys (defined in Definitions \ref{def:sp-minima} and \ref{def:sp-valley}) even on wide networks.
We summarize our constribution in Table~\ref{table:res}.
We say a two-layer neural net to be an SD net if it has a sparse first layer and a dense final layer, and to be a SS net if it has a sparse first layer and a sparse final layer.

\begin{savenotes}
	\begin{table}[t]
		\renewcommand{\arraystretch}{1}
		\centering
		\begin{tabular}{|c|c|c|}
			\hline
			& Spurious valleys for linear nets & Spurious valleys for non-linear nets \\
			\hline
			\ D nets & No (\cite{venturi2019spurious}) & No (\cite{venturi2019spurious, li2018benefit}) \\
			\hline
			SD nets & No (Theorem \ref{thm:two-layer-linear}) & No (\cite{nguyen2018loss,nguyen2019connected} and Theorem \ref{thm:non-linear})  \\
			\hline
			SS nets & Yes (Theorem \ref{thm:local-valley}) & Yes (Theorem \ref{thm:local-valley}) \\
			\hline
		\end{tabular}
		\caption{Summary of our results and comparison of existing results for two-layer linear/non-linear networks. D nets: dense networks; SD nets: sparse networks with a sparse first layer and a dense final layer; SS nets: sparse networks with a sparse first layer and a sparse final layer. 
		Some works adopt analogous definitions instead of valleys. We view all of them as valleys for brevity.
		Additionally, for SD nets, the results of linear networks hold for some special structures, and the results of non-linear networks hold for wide SD nets.}
		\label{table:res}
	\end{table}
\end{savenotes}

\begin{itemize}
    \item For a linear SD net, we show that sparse networks can have a benign landscape as dense linear networks under some special cases (in Theorem \ref{thm:two-layer-linear}). 
    Specifically, we show no spurious valleys if one of the following conditions holds: each sparse pattern is over-parameterized, or data are orthogonal with non-overlapping filters, or the target is scalar. 
    \item For a wide non-linear SD net, we show the consistency between sparse and dense networks (in Theorems \ref{thm:infinite-dim-non-linear} and \ref{thm:non-linear}). We show no spurious valleys for polynomial activations when the hidden layer width is larger than corresponding upper intrinsic dimensions \cite{venturi2019spurious}, and for real analytic activations when the hidden layer width is larger than the number of training samples. This is the same as dense networks.
    \item For a linear/non-linear SS net, we identify the difference between sparse and dense networks by constructively proving the existence of spurious valleys and spurious minima (in Theorem \ref{thm:local-valley}) with common activations.
\end{itemize}

Finally, we conduct numerical experiments under linear/non-linear activations to verify our theoretical results. 
Though we mainly focus on two-layer networks, \textit{we also generalize some results to deep networks.}
For example, we show no spurious valleys for deep linear networks with all sparse layers under scalar output (in Theorem \ref{thm:linear-deep}). 
Moreover, we prove no spurious valleys with non-empty interiors for deep sparse networks with a dense wide final layer (in Theorem \ref{thm:deep-non-linear}). 
In summary, we hope that our work contributes to a better understanding of the landscape of sparse networks and brings insights for future research.

\subsection{Related Work} \label{sec:related}
Our work is strongly related to and inspired by a line of recent work that analyzes the landscape of neural networks.

\paragraph{Landscape of neural networks} 
The landscape analysis of neural networks was a popular topic in the early days of neural net
research; see, e.g., \citet{bianchini1996optimal} for an overview. 
One notable early work \cite{baldi1989neural} proved that shallow linear neural networks do not have bad local minima. 
\citet{kawaguchi2016deep} generalized this result to deep linear neural networks.
However, the situation is more complicated when nonlinear activations are introduced \cite{yun2018small, he2020piecewise,  goldblum2019truth}. 
Many works \cite{zhou2017critical, yun2018small, safran2018spurious, he2020piecewise, ding2019sub} show
that spurious local minima can happen even in a two-layer network with nonlinear activations.
Despite the existence of spurious minima, over-parameterized networks can still exhibit some nice geometrical properties.
\citet{freeman2019topology} considered the topological properties of the sublevel sets. 
\citet{venturi2019spurious} showed the absence of spurious valleys for polynomial activations under a wide hidden layer, and proved the existence of spurious valleys for non-polynomial activation functions if the hidden layer is narrow. 
\citet{nguyen2019connected} showed that a pyramidal DNN has no spurious ``valleys'' for one wide layer and strictly monotonic activations.
Moreover, \citet{li2018benefit} used a different notion named ``spurious basin'' to show that deep fully connected networks with any continuous activation admit no spurious basins.
\citet{nguyen2018loss} also showed no spurious ``valleys'' defined by strict sublevel sets for general sparse networks with plenty of connections to each output neuron and strictly monotonic activations.

Some works \cite{mehta2021loss, kunin2019loss} also focus on the effect of regularization in deep linear networks.
\citet{mehta2021loss} employed novel algebraic symmetries that result in ``flat'' critical manifolds, and analyzed how $\ell_2$-regularization breaks these symmetries to produce isolated critical points.
In this work, we do not apply regularization, because linear networks and wide non-linear networks already have a benign landscape based on previous work.  We consider using as little as possible in the component of the loss function to see the influence of sparsity. 
Moreover, \citet{mehta2021loss} did not distinguish between spurious minima and global minima, which we think the theory still needs to handle.
  
\paragraph{Sparse networks}
Sparse networks \cite{han2015learning, zhu2017prune, liu2018rethinking,frankle2018lottery, han2016deep} have a long history,
and they have gained many recent interests due to their potential in on-device AI (i.e., deploying AI models on small devices).
In practice, sparse networks are mainly produced by the network pruning technique \cite{han2015learning, gale2019state, lee2018snip, liu2018rethinking, frankle2018lottery, yu2018nisp, he2017channel, mozer1989using, morcos2019one, luo2017thinet}, to name a few. 
There are many variants of pruning methods and surveys of this literature, e.g., \citet{hoefler2021sparsity, mishra2020survey}.
Pruning can lead to a reduction in storage and model runtime. 
The performance is usually maintained by retraining the pruned network. 

Most research efforts are spent on the empirical aspects. 
\citet{frankle2018lottery} recommended reusing the sparsity pattern found through pruning and training a sparse network from the same initialization as the original training (i.e., ``lottery'') to obtain comparable performance and avoid a bad solution. 
Moreover, several recent works also give abundant methods for choosing weights or sparse network structures while achieving comparable performance \cite{lee2018snip, molchanov2017variational, louizos2018learning, carreira2018learning,liu2022the}.
There are some theoretical investigations on representation
power. Several works \cite{malach2020proving, pensia2020optimal} proved that an over-parameterized neural network contains a subnetwork with roughly the same accuracy as the target network, providing guarantees for existence of ``good'' sparse subnetworks.

The closest work to ours is probably \citet{evci2019difficulty}, who showed empirically that bad local minima can appear in a sparse network, but this work did not provide any theoretical result. 
Our work provides a partial theoretical justification for this phenomenon: we show that spurious valleys can exist in a sparse network.

\subsection{Organization}
This article is organized as follows. 
We formally define our setting and notation in Section \ref{sec:notation}. 
Then we begin with the analysis for linear SD nets in Section \ref{section:general-linear}.
And we shed light on non-linear SD nets in Section \ref{sec:non-linear}.
In Section \ref{subsec:conter}, we give a counterexample for SS nets.
We verify our theorems through experiments in Section \ref{sec:exp}. 
Finally, we report conclusion in Section \ref{sec:conclusion}.
\section{Preliminaries}\label{sec:notation}
\paragraph{Notation} We use bold-faced letters (e.g., $\bm{w}$ and $\bm{a}$) to denote vectors, and  capital letters (e.g., ${W}=[w_{i j}]$ and ${A}=[a_{i j}]$) for matrices. 
We sometimes use $(W)_{i, \cdot}$ and $(W)_{\cdot, j}$ as the $i$-th row and $j$-th column of a matrix ${W}$, and $(\bm{a})_i$ as the $i$-th entry of a vector $\bm{a}$, if no ambiguity.
We denote $\widetilde{W}$ as a sparse matrix of $W$ (which will be illustrated in detail), and $W^+$ as the pseudoinverse matrix of $W$.
We abbreviate 
% $[k: n]=\{k, \ldots, n\}$ and
$[n]=\{1, \ldots, n\}$. We let $\bm{e}_i$ as the standard $i$-th unit vector, $\bm{1}_n$ and $\bm{0}_n$ respectively denote the all-ones and all-zeros vectors in $\mathbb{R}^n$.
The range of a real-value function $\sigma(\cdot)$ is $\sigma(\mathbb{R}) = \{\sigma(x): x \in \mathbb{R}\}$, and the number of elements in a set $\mathcal{S}$ is $|\mathcal{S}|$. We denote $\mathrm{span}(\mathcal{S})$ as the linear span (or the linear hull) of a set $\mathcal{S}$ of vectors.

\paragraph{Problem setup} Given training samples  $\{(\bm{x}_i,\bm{y}_i)\}_{i=1}^n\subset \mathbb{R}^{d_x}\times \mathbb{R}^{d_y}$,
we form the data matrices $X=[\bm{x}_1, \ldots, \bm{x}_n] \in \mathbb{R}^{d_x \times n}$ and $Y =[\bm{y}_1, \ldots, \bm{y}_n] \in \mathbb{R}^{d_y \times n}$, respectively. 
We mainly focus on one-hidden-layer neural networks with squared loss, while some results can also be applied to general convex loss.
Then the objective is \footnote{Adding bias is equivalent to adding a row of vector $\bm{1}_n^\top$ to $X$ and our setting also includes sparse bias in the first layer. Hence we do not distinguish bias terms in the subsequent content.}
\begin{equation}\label{eq loss func}
\min_{U, W} \; L(U, W) := \frac{1}{2}\left\|U\sigma(W X)-Y\right\|_F^2,
\end{equation}
where $W =[\bm{w}_1, \ldots, \bm{w}_p]^\top \in \mathbb{R}^{p\times d_x}$, $U=[\bm{u}_1, \ldots, \bm{u}_p] \in \mathbb{R}^{d_y \times p}$. Here $ p $ represents the width of the hidden layer and $\sigma(\cdot)$ is a continuous element-wise activation function.

Generally, when applied to sparse patterns after weight pruning or masking, many weights become zero and would not be updated in training.
Our goal is to study the landscape of sparse structures after weight pruning. 
Thus, we pay much attention to the sparse structures, rather than the underlying pruning methods.

We begin with the SD nets (sparse-dense networks). We denote each pattern for each hidden weight $\bm{w}_i$ as $\bm{m}_i\in\{0,1\}^{d_x}, i \in [p]$ and the first layer weights have totally $s$ different patterns $\bm{m}_1^*, \ldots, \bm{m}_s^* \neq \bm{0}$. 
We recombine weights and data with the same pattern as $\mathcal{S}_i=\{j:\bm{m}_j=\bm{m}_i^*\}, i\in[s]$. 
Then the objective of the sparse one-hidden-layer network becomes
\begin{equation} \label{eq:objective}
\min_{U, \widetilde{W}} L(U, \widetilde{W}) := \frac{1}{2}\left\|\sum_{i=1}^{s} U_i\sigma(W_i Z_i)-Y \right\|_F^2 = \frac{1}{2}\left\|\left(U_1, \ldots, U_s \right) \sigma \left[\left(\begin{smallmatrix}
W_1 & \cdots & \bm{0} \\
\vdots & \ddots & \vdots \\
\bm{0} & \cdots & W_s \\
\end{smallmatrix}\right)
\left(\begin{smallmatrix}
Z_1\\ \vdots \\ Z_s 
\end{smallmatrix}\right)\right]-Y \right\|_F^2,
\end{equation}
where we view sparse hidden-layer weight as a block diagonal matrix with elements $W_i = [\bm{w}_{j}, j\in\mathcal{S}_i]^\top \in \mathbb{R}^{p_i\times d_i}$, and we also duplicate and rearrange input data as $Z_i = [X_{j,\cdot}, (\bm{m}^*_i)_j=1] \in \mathbb{R}^{d_i\times n}$, and separating $U$ as $U_i = [\bm{u}_{j}, j\in\mathcal{S}_i]^\top \in \mathbb{R}^{d_y \times p_i}$. Here $i \in [s], p_i = |\mathcal{S}_i|, d_i=\bm{1}_{d_x}^\top\bm{m}^*_i$, $p=\sum_{i=1}^{s}p_i$ (see Figure \ref{fig:sketchmap} for example).

\begin{figure}[!t]
	\centering
	\includegraphics[width=0.7\linewidth]{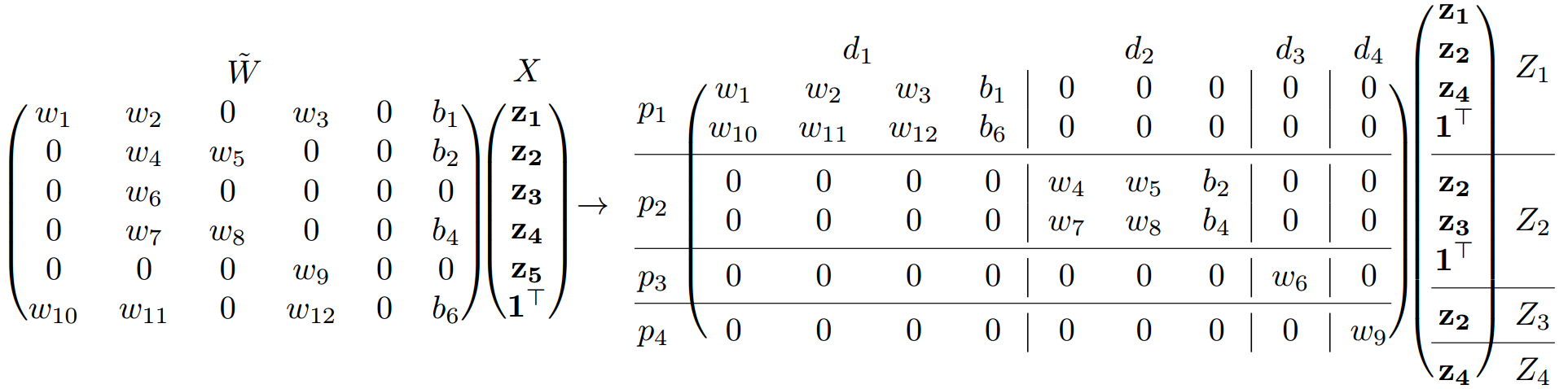}
	\caption{An example of sparse weight transformation in Eq.~\eqref{eq:objective}.}
	\label{fig:sketchmap}
\end{figure}

Note that there may be useless connections and nodes, such as a node with zero out-degree which can be retrieved and excluded from the final layer to the first layer, and other cases are shown in \ref{app:useless-connect}. Thus we do not consider them and assume the sparse structure is \textit{effective} by Definition \ref{defn:eff}, meaning that each neuron has a potential contribution to the network output.

\begin{defn}[\textit{Effective neuron and sparse network}]\label{defn:eff}
	A neuron is effective if it appears at least in one directed path from one input node to one output node.
	A sparse network is effective if each neuron (including input and output neurons) is effective.
\end{defn}

Previous works mainly pay attention to local minima (Definition \ref{def:sp-minima}) and valleys (Definition \ref{def:sp-valley}) for describing the landscape of neural networks. No spurious valleys imply the absence of spurious strict minima \cite{freeman2019topology,venturi2019spurious}. 
Local search methods may get stuck in a spurious valley as well as a spurious strict minimum. 
Thus, understanding the existence of such 
objectives in the landscape would help understand
the difficulty of training (sparse) neural networks.

\begin{defn}[\textbf{Spurious minimum}]\label{def:sp-minima}
	We say a point to be a spurious (strict) minimum if it is a (strict) local minimum but not a global minimum. 
	Here, $f(\bm{\theta}_0)$ is a local minimum, that is, there exists $\epsilon > 0$ such that $f(\bm{\theta}_0) \leq f(\bm{\theta}),\forall \bm{\theta} \colon \left\|\bm{\theta}-\bm{\theta}_0\right\| \leq \epsilon$. And $f(\bm{\theta}_0)$ is a strict local minimum, that is, there exists $\epsilon > 0$ such that $f(\bm{\theta}_0) < f(\bm{\theta}),\forall \bm{\theta} \colon \left\|\bm{\theta}-\bm{\theta}_0\right\| \leq \epsilon, \bm{\theta} \neq \bm{\theta}_0$.
\end{defn}

\begin{defn}[\textbf{Spurious valley}]\label{def:sp-valley}
	Given an $ \alpha \in\mathbb{R} $, we define the $\alpha$-sublevel set of $f(\bm{\theta})$ as $\Omega_f(\alpha) = \{\bm{\theta}\in \mathrm{dom}(f) \colon f(\bm{\theta}) \leq \alpha \}$. A spurious valley $\mathcal{T}$ is a connected component of a sublevel set $\Omega_f(\alpha)$ which can not approach the infimum of $f(\bm{\theta})$, that is, $\inf_{\bm{\theta} \in \mathcal{T}} f(\bm{\theta}) > \inf_{\bm{\theta} \in \mathrm{dom}(f)} f(\bm{\theta})$.
\end{defn}

\section{Sparse-dense Linear Networks} \label{section:general-linear}
We begin with linear activation ($\sigma(z)=z$) as a warm-up case to understand the landscape of neural networks. 
Previous works have proven that there are no spurious minima \cite{kawaguchi2016deep,lu2017depth} and no spurious valleys \cite{venturi2019spurious} for dense linear networks.
We will inherit their descriptions of landscape to show that under certain conditions, we could guarantee no spurious valleys for SD linear networks.

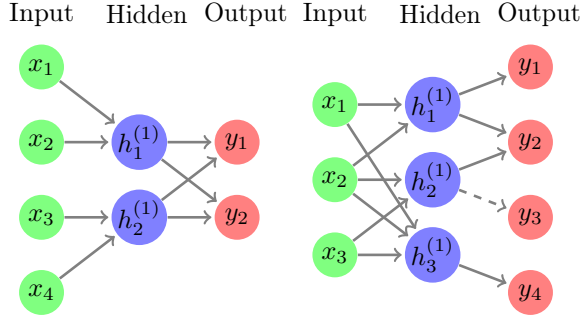
\begin{figure}[!t]
	\centering
	\def\layersep{1.3cm}
	\begin{tikzpicture}[shorten >=1pt,->,draw=black!50, node distance=\layersep]
	\tikzstyle{every pin edge}=[<-,shorten <=1pt]
	\tikzstyle{neural}=[circle,fill=black!25,minimum size=17pt,inner sep=0pt]
	\tikzstyle{input neural}=[neural, fill=green!50];
	\tikzstyle{output neural}=[neural, fill=red!50];
	\tikzstyle{hidden neural}=[neural, fill=blue!50];
	\tikzstyle{annot} = [text width=2.5em, text centered]
	% Draw the input layer nodes
	\foreach \name / \y in {1,...,4}
	% This is the same as writing \foreach \name / \y in {1/1,2/2,3/3,4/4}
	\node[input neural] (I-\name) at (0,-\y) {$x_{\y}$};
	% Draw the hidden layer nodes
	\foreach \name / \y in {1,...,2}
	\path[yshift=-1cm]
	node[hidden neural] (H-\name) at (\layersep,-\y cm) {$h^{(1)}_{\y}$};
	% Draw the output layer node
	\foreach \name / \y in {1,2}
	\path[yshift=-1cm]
	node[output neural] (O-\name) at (2*\layersep,-\y cm) {$y_{\y}$};
	% Connect every node in the input layer with every node in the
	% hidden layer.
	\foreach \source in {1,2}
	\foreach \dest in {1}
	\path (I-\source) edge [line width=1pt](H-\dest);
	\foreach \source in {3,4}
	\foreach \dest in {2}
	\path (I-\source) edge [line width=1pt](H-\dest);
	\foreach \source in {1,2}
	\foreach \dest in {1,2}
	\path (H-\source) edge [line width=1pt] (O-\dest);
	% Annotate the layers
	\node[annot,above of=H-1, node distance=1.7cm] (hl) {Hidden};
	\node[annot,left of=hl] {Input};
	\node[annot,right of=hl] {Output};
	% Draw the input layer nodes
	\foreach \name / \y in {1,...,3}
	% This is the same as writing \foreach \name / \y in {1/1,2/2,3/3,4/4}
	\path[yshift=-0.5cm]
	node[input neural] (I-\name) at (3*\layersep,-\y) {$x_{\y}$};
	% Draw the hidden layer nodes
	\foreach \name / \y in {1,...,3}
	\path[yshift=-0.5cm]
	node[hidden neural] (H-\name) at (4*\layersep,-\y cm) {$h^{(1)}_{\y}$};
	% Draw the output layer node
	\foreach \name / \y in {1,...,4}
	\path[yshift=0cm]
	node[output neural] (O-\name) at (5*\layersep,-\y cm) {$y_{\y}$};
	% Connect every node in the input layer with every node in the hidden layer.
	\path (I-1) edge [line width=1pt](H-1);
	\path (I-2) edge [line width=1pt](H-1);
	\path (I-2) edge [line width=1pt](H-2);
	\path (I-3) edge [line width=1pt](H-2);
	\path (I-1) edge [line width=1pt](H-3);
	\path (I-2) edge [line width=1pt](H-3);
	\path (I-3) edge [line width=1pt](H-3);
	\path (H-1) edge [line width=1pt](O-1);
	\path (H-1) edge [line width=1pt](O-2);
	\path (H-2) edge [line width=1pt](O-2);
	\path (H-2) edge [line width=1pt, dashed](O-3);	
	\path (H-3) edge [line width=1pt](O-4);	
	% Annotate the layers
	\node[annot,above of=H-1, node distance=1.2cm] (hl) {Hidden};
	\node[annot,left of=hl] {Input};
	\node[annot,right of=hl] {Output};
	\end{tikzpicture}
	\caption{Sparse network without (left) / with (right) overlapping filters in the first layer.} \label{fig:1}
\end{figure}

\begin{thm}\label{thm:two-layer-linear}
	Under \textbf{either} of the following conditions, any \textit{effective} SD linear network has no spurious valleys:
	1) $p_i \geq d_i, \forall i \in [s]$;
	2) $Z_i Z_j^\top = \bm{0}, \forall i \neq j \in [s]$;
	3) $d_y=1$.
\end{thm}

The proof of Theorem \ref{thm:two-layer-linear} leaves in \ref{app:shallow-linear}. 
We prove the results by constructing a non-increasing path to approach the infimum of the loss function (i.e., the zero loss).
That is, we show {\PP} below under the conditions in Theorem \ref{thm:two-layer-linear}. Here we use the fact that {\PP} implies the absence of spurious valleys \cite[Lemma 2]{venturi2019spurious}.

\begin{defn}[\PP]
	We say the landscape of an objective $f(\bm{\theta})$ satisfies {\PP} if there exists a non-increasing continuous path from any initialization to approach the infimum. 
	Formally, given any initial parameter $\bm{\theta}_0 \in \mathrm{dom}(f)$, there exists a continuous path $ \bm{\theta} \colon t \in [0, 1) \to \bm{\theta}(t) \in \mathrm{dom}(f)$
	such that:
	(a) $\bm{\theta}(0) = \bm{\theta}_0$; (b) $\lim_{t \to 1}f(\bm{\theta}(t)) = \inf_{\bm{\theta}} f(\bm{\theta})$; (c) the function $f(\bm{\theta}(t)) $ in $t \in [0, 1)$ is non-increasing.
\end{defn}

The first condition in Theorem \ref{thm:two-layer-linear} can be viewed as over-parameterization in the sparse networks for each sparse pattern, and the second condition includes two parts: the patterns are non-overlapping (as the left graph of Figure \ref{fig:1} depicts) and the input data is orthogonal, showing that structured training data matched with a sparse network can preserve the benign landscape. 
The third condition states that the output dimension is one.
From Theorem \ref{thm:two-layer-linear}, we could see some special sparse linear networks still have benign landscapes as dense linear networks, which motivates researchers to discover or design efficient sparse networks.  

\subsection{Extensions} \label{sec:exten-linear}
We generalize Theorem \ref{thm:two-layer-linear} to deep linear networks combined with previous works, and the proof can be found in \ref{app:ext-linear}. 
The intuition is that deep linear networks have similar landscape properties as the shallow case \cite{yun2018small, lu2017depth}. 
However, except for the scalar output case which we have solved, understanding the landscape of an arbitrary deep sparse network is still complicated. 
 
\begin{thm}\label{thm:linear-deep}
	Under \textbf{either} of the following conditions, any \textit{effective} deep linear neural network with a sparse first layer does not have spurious valleys:
    1) $p_i \geq d_i, \forall i \in [s]$. 2) $d_y=1$. Moreover, the result under assumption 2) holds for sparse networks with \textbf{all} sparse layers.
\end{thm}

\section{Sparse-dense Non-linear Networks} \label{sec:non-linear}

Previous work \cite{yun2018small, he2020piecewise} has already certified the intrinsic difference between linear and non-linear activations. 
However, the landscape of neural networks still has benign geometric properties \cite{venturi2019spurious, nguyen2018loss, nguyen2019connected} in the scope of no spurious valleys, if the network is wide enough. In this section, we verify that a wide SD net with non-linear activations still possesses a benign landscape.
We provide no spurious valleys for polynomial activations in Subsection \ref{subsec:polyax}, for general real analytic activations in Subsection \ref{subsec:saa}, and defer the proofs to \ref{app:non-linear}.

\subsection{Polynomial Activations} \label{subsec:polyax}
For some simple activations, such as polynomial functions, we could view them as the linear activation after non-linear mappings. 
Specifically, we give an example of quadratic activation when $d_x=2$, $\bm{x}=(x_1,x_2)$ and $\bm{w}=(w_1,w_2,b)$. 
If we define $\bm{\psi}(\bm{w}) = (w_1^2, w_2^2, 2w_1w_2, 2bw_1, 2bw_2, b^2), \bm{\phi}(\bm{x})= (x_1^2, x_2^2, x_1x_2, x_1, x_2, 1)$,
then we can view $\sigma(w_1x_1+w_2x_2+b)=(w_1x_1+w_2x_2+b)^2=\langle \bm{\psi}(\bm{w}), \bm{\phi}(\bm{x}) \rangle$, showing that we can convert polynomial activation into linear activation after we project original data and weight to mapping spaces.

More abstractly, we define $V_{\sigma}(X) := \mathrm{span}\left(\left\{\sigma(\bm{w}^\top X), \bm{w}\in\mathbb{R}^{d_x}\right\}\right)$ and the upper intrinsic dimension \cite{venturi2019spurious} as $ \dim^*(\sigma, X) :=\mathrm{dim} \left(V_{\sigma}(X)\right)$.
Then from the above example, we could see $\dim^*(\sigma, X) \leq 6$ for quadratic activation with $d_x=2$ because $ \bm{\psi}(\bm{w})\in\mathbb{R}^6$. 
Therefore, we can still regard the non-linear activations (particularly for polynomial activations) as the linear activation on feature mappings $\bm{\psi}(\bm{w})$ and $\bm{\phi}(\bm{x})$.

\begin{thm} \label{thm:infinite-dim-non-linear}
	For any continuous activation function $ \sigma (\cdot) $, suppose $ \dim^*(\sigma, X) < +\infty $. Then any effective SD net has no spurious valleys if $ p_i \geq \dim^*(\sigma, Z_i), \forall i \in[s]$.
\end{thm}

\begin{remark}
	The condition $\dim^*(\sigma, X) < +\infty$ has also appeared in \cite{venturi2019spurious}, but we extend it to the sparse setting. 
	Moreover, if $\sigma(z)=\sum_{k=0}^{t}a_k z^k$, then $\dim^*(\sigma, X) = O(d^t)$ (e.g., see \citet[Corollary 10]{venturi2019spurious}). 
	Therefore, we obtain no spurious valleys when the hidden width $p = \sum_{i=1}^s p_i =  \Omega(\sum_{i=1}^{s}d_i^t)$. 
	For non-overlapping filters, if we choose $d_i = d/s$ for total $s$ patterns and $t\geq 3$, then $p = \Omega(s(d/s)^t)=\Omega(d^t/s^{t-1})$, giving less hidden-layer width requirement compared to $\Omega(d^t)$ in the dense setting from \citet[Corollary 10]{venturi2019spurious}.
\end{remark}

\subsection{Real Analytic Activations} \label{subsec:saa}

When applied to general non-polynomial activations, particularly on real analytic activations, \citet{nguyen2018loss} provided results for sparse networks but adopted strictly increasing activations and strict sublevel sets to define spurious valleys. \citet{li2018benefit} used a different notion of ``spurious basin'' to show that deep fully-connected networks with any continuous activation admit no spurious basins, but did not apply it to sparse networks. 
Yet, we provide no spurious valleys for SD nets with real analytic activations. We show the comparison in Table~\ref{table:res2}.

\begin{savenotes}
	\begin{table}[!t]
		\renewcommand{\arraystretch}{1}
		\centering
		\begin{tabular}{|c|c|c|c|}
			\hline
			& \citet{nguyen2018loss} & \citet{li2018benefit} & Theorem \ref{thm:non-linear} \\
			\hline
			Sparse NN & $\checkmark$ & $\times$ & $\checkmark$ \\
			\hline
			Non-increasing $\sigma$ & $\times$ & $\checkmark$ & $\checkmark$ \\
			\hline
		\end{tabular}
		\caption{Comparison of results for two-layer networks with real analytic activations.}
		\label{table:res2}
	\end{table}
\end{savenotes}

Before introducing our results, we need some assumptions as previous works do.
\begin{assume}\label{ass:data}
	1) $\forall i \neq j \in [n], \forall k \in [d_x], \left|(\bm{x}_{i})_k\right| \neq \left|(\bm{x}_{j})_k \right| \neq 0$. 
	2) Except the last hidden-layer weight and all biases, no hidden-layer weights are all pruned, i.e., $\widetilde{W}_{i,\cdot} \neq \bm{0}_{d_x}^\top, \forall i \in[p]$.
\end{assume}

\begin{assume}\label{ass:ac}
	Activation function $\sigma(\cdot)$ is real analytic, and there exist $n$ distinct non-negative integers $l_1, \dots , l_{n}$ which form an arithmetic sequence\footnote{That is, $l_i=l_1+(i-1)(l_1-l_0), \forall i \in [n]$.}, such that
	$\sigma^{(l_i)}(0)\neq 0, \forall i \in [n]$,	where $\sigma^{(k)}(0)$ denotes the $k$-th order derivative of $\sigma(\cdot)$ at zero.
\end{assume}

The first part of Assumption \ref{ass:data} shows that training data have different features in each dimension, which also appears in previous works \cite{li2018benefit}.
And this can be easily achieved if an arbitrarily small perturbation of data is allowed. 
The second part of Assumption \ref{ass:data} excludes useless patterns which only employ bias feature $\bm{1}_n^\top$ and do not encode any effective features among training data. 
Assumption \ref{ass:ac} shows the least non-linearity requirement of activation at the origin, which can be satisfied for Sigmoid, Tanh, and Softplus (a smooth approximation of ReLU).  
Now we state two-layer SD networks admit no spurious valleys in the over-parametrized regime.

\begin{thm} \label{thm:non-linear}
	Under Assumptions \ref{ass:data} and \ref{ass:ac}, if the hidden-layer width $ p \geq n $, we have $\mathrm{rank}(\sigma(\widetilde{W}X))=n, a.s.$, and any \textit{effective} SD non-linear network has no spurious valleys.
\end{thm}

The key proof idea is to show that we could always make hidden-layer $\sigma(\widetilde{W}X)$ have full column rank with invariant loss. Otherwise, we could calculate $l_i$-th order derivative of $\sigma(\widetilde{W}X)$ at a specific row, and show the contradiction from Assumptions \ref{ass:data} and \ref{ass:ac}.
It is interesting to note that Theorem \ref{thm:non-linear} explains no substantial difference in the scope of spurious valleys when sparsity is only introduced in the first layer.
Additionally, the requirement of the large width ($p\geq n$) also appears in \cite{venturi2019spurious,li2018benefit,nguyen2018loss,nguyen2019connected}. 

\subsection{Extension}
We extend Theorem \ref{thm:non-linear} to deep sparse networks. 
Due to the complex stack structure of deep sparse networks, we could only show no spurious valleys with non-empty interiors. Such constraint points out that spurious valleys, if exist, are certainly degenerated.

\begin{thm} \label{thm:deep-non-linear}
	Under Assumptions \ref{ass:data} and \ref{ass:ac}, any \textit{effective} deep sparse neural network has no spurious valleys with non-empty interiors if the last hidden-layer width $ p \geq n $ and each output neuron has at least $n$ connected neurons in the last hidden layer.
\end{thm}

Analogous to the two-layer case, the main proof idea is to show the almost surely full rank property of the output matrix of each hidden layer. 
Moreover, previous works also need to refine the definition of the spurious valley to deep networks, such as the spurious basin \cite{li2018benefit}, and the spurious ``valley'' defined by \textit{strict} sublevel sets \cite{nguyen2018loss, nguyen2019connected}.

\section{Sparse-sparse Networks} \label{subsec:conter}
Finally, we consider the remaining setting when sparsity is introduced in both layers, i.e., SS nets (sparse-sparse networks). 
In the following, we will show the necessity of a dense (or $n$ connected) final layer for general continuous activations through a concrete example with spurious valleys. 
Furthermore, in our construction, the spurious valley is a global minimum set of a sub-network \footnote{Here, the \textit{sub-network} means a network by removing some connections from the original network.} of the original sparse network.

Since no spurious valleys appear in dense wide (two-layer) networks \cite{venturi2019spurious, nguyen2018loss, nguyen2019connected},
our result indicates that sparsity can provably deteriorate the landscape.

\begin{thm} \label{thm:local-valley}
	If the continuous activation satisfies $\sigma(0)=0$ and $\sigma(\mathbb{R}) \neq \{0\}$, and the hidden width $p \geq 3$, then there exists an \textit{effective} SS network, which has a spurious valley. Additionally, this spurious valley corresponds to a global minimum set of a sub-network of the original sparse network. 
\end{thm}

\begin{proof}
    The proof mainly utilizes the relationship among sparse connections.
	We consider a sparse network with the sub-structure in the right graph of Figure \ref{fig:1}, and the objective becomes:
	\begin{equation}\label{eq:cx-obj}
	\begin{aligned}
	\min_{\substack{\bm{\theta}=\left(w_1,\dots, w_8\right) \\ \bm{\theta}_1=\left(\bm{v}_1,\bm{v}_2,\bm{v}_3, V\right)}} L(\bm{\theta}, \bm{\theta}_1) & = \frac{1}{2}\left\| \left(\begin{smallmatrix}
	w_1 & 0 & \bm{0}^\top \\ w_2 & w_3 & \bm{0}^\top \\ 0 & w_4 & \bm{0}^\top \\ \bm{0} & \bm{0} & V
	\end{smallmatrix}\right)\sigma\left[
	\left(\begin{smallmatrix}
	w_5 & w_6 & 0 \\ 0 & w_7 & w_8 \\ \bm{v}_1 & \bm{v}_2 & \bm{v}_3 
	\end{smallmatrix}\right)X\right]-Y
	\right\|_F^2 \\
	&= \frac{1}{2}\left\|
	\left(\begin{smallmatrix}
	w_1\sigma(w_5) & w_1\sigma(w_6) &  0 \\ 
	w_2\sigma(w_5) & w_2\sigma(w_6)+w_3\sigma(w_7) & w_3\sigma(w_8) \\ 
	0 & w_4\sigma(w_7) & w_4\sigma(w_8) \\
	V\sigma(\bm{v}_1) & V\sigma(\bm{v}_2) & V\sigma(\bm{v}_3) \\
	\end{smallmatrix}\right)-Y
	\right\|_F^2,
	\end{aligned}
	\end{equation}
	where we assume $X = I_3$ and $Y = \left(\begin{smallmatrix}
	y_1 & y_1   & 0    \\
	y_2 & y_2+y_3 & 0    \\
	0 &    0   & y_4  \\
	\bm{*} & \bm{*} & \bm{*} \\
	\end{smallmatrix}\right)$ with $y_3>4y_4>4y_1>0$, $y_2>0$. 
	
	Because the variables $\bm{v}_i$ and $V$ are unrelated to the $w_i$, we only need to consider the $w_i$ while assuming the $\bm{v}_i$ and $V$ can always obtain zero loss for the remaining rows in $Y$.
	Therefore, in the following, we only consider the parameter $\bm{\theta}$.
	
	Now we consider the sublevel set $\mathcal{T}:=\{\bm{\theta}: L(\bm{\theta}) \leq y_4^2 / 2\} \cap \{w_4=0\}$.
	Noting that when $w_4=0$, $L(\bm{\theta})\geq y_4^2/2$, we obtain $L(\bm{\theta})=y_4^2/2$ and
	$\left(\begin{smallmatrix}
	w_1 & 0 \\ w_2 & w_3
	\end{smallmatrix}\right)
	\left(\begin{smallmatrix}
	\sigma(w_5) & \sigma(w_6) & 0 \\ 0 & \sigma(w_7) & \sigma(w_8) \\
	\end{smallmatrix}\right) = \left(\begin{smallmatrix}
	y_1 & y_1 & 0   \\
	y_2 & y_2+y_3 & 0 \\
	\end{smallmatrix}\right), \forall \bm{\theta} \in \mathcal{T}$, leading to that 
	\begin{equation*}\label{eq:T}
	\begin{aligned}
		\mathcal{T} {=} \bigg\{ \bm{\theta}(a,b) \bigg| & 
		\left(\begin{smallmatrix}
		w_1 & 0 \\ w_2 & w_3 \\ 0 & w_4
		\end{smallmatrix}\right) {=} \left(\begin{smallmatrix}
		y_1 a & 0 \\ y_2 a & y_3b \\ 0 & 0
		\end{smallmatrix}\right), \left(\begin{smallmatrix}
		\sigma(w_5) & \sigma(w_6) & 0 \\ 0 & \sigma(w_7) & \sigma(w_8) \\
		\end{smallmatrix}\right) {=} \left(\begin{smallmatrix}
		1/a & 1/a & 0 \\ 0 & 1/b & 0 \\
		\end{smallmatrix}\right), a, b {\neq} 0 \bigg\}.
	\end{aligned}
	\end{equation*}
	Since $\sigma(\mathbb{R}) \neq \{0\}$, we have $\mathcal{T}\neq \emptyset$. We choose a connected component of $\mathcal{T}_0 \in\mathcal{T}$ with $a, b > 0$ or $a,b<0$ depending on  $\sigma(\mathbb{R})$.
	For any small disturbances $\delta_i, i\in[8]$,
	\begin{equation*}
	\begin{aligned}
	& 2 L(w_1+\delta_1 ,\dots, w_8+\delta_8) = \left\| \left(\begin{smallmatrix}
	* & * & * \\
	* & * & (w_3+\delta_3)\sigma(w_8+\delta_8) \\
	* & \delta_4\sigma(w_7+\delta_7) & \delta_4\sigma(w_8+\delta_8)-y_4
	\end{smallmatrix}\right) \right\|_F^2 \\
    & \stackrel{(i)}{\geq}  (w_3+\delta_3)^2\sigma(w_8+\delta_8)^2 + \delta_4^2\sigma(w_7+\delta_7)^2 +  \left(\delta_4\sigma(w_8+\delta_8)-y_4 \right)^2 \\
    & \stackrel{(ii)}{\geq}  2\left|(w_3+\delta_3)\sigma(w_7+\delta_7)\delta_4\sigma(w_8+\delta_8)\right|- 2y_4\delta_4\sigma(w_8+\delta_8)+y_4^2 \stackrel{(iii)}{\geq} y_4^2,
	\end{aligned}
	\end{equation*}
	where we only consider the last column and row of loss in $(i)$, and $(ii)$ uses the arithmetic and geometric means inequality, $(iii)$ is derived from small permutation when we choose $|\delta_3| \leq |w_3|/2$ and $|\delta_7| $ small enough, such that $|\sigma(w_7+\delta_7)-\sigma(w_7)|\leq |\sigma(w_7)|/2$ since $\sigma(\cdot)$ is continues and $\sigma(w_7)\neq 0$, leading to
	$|(w_3+\delta_3)\sigma(w_7+\delta_7)|\geq |(w_3/2)\sigma(w_7)/2|= y_3/4>y_4>0$.
	
	Thus, each $\bm{\theta}\in\mathcal{T}_0$ is a local minimum, and $\mathcal{T}_0$ is connected. Moreover, we could see if $\delta_4 \neq 0$, from $(iii)$, the equality only holds when $\sigma(w_8+\delta_8)=0$. Then from $(i)$, we obtain that
	$ 2 L(w_1+\delta_1 ,\dots, w_8+\delta_8) \geq \delta_4^2\sigma(w_7+\delta_7)^2 + y_4^2 \geq \delta_4^2\sigma(w_7)^2/4 + y_4^2>y_4^2$.
	Hence, a connected component of a sublevel set $\{\bm{\theta}: L(\bm{\theta}) \leq y_4^2/2 \}$ that contains $\mathcal{T}_0$ would never reach some point $\bm{\theta}'$ with $w_4'\neq 0$ and $L(\bm{\theta}') \leq y_4^2/2$.
	
    Finally, we show that $\mathcal{T}_0$ is not a global minimum set. We choose a point $\tilde{\bm{\theta}}(a)$ that
    $\left(\begin{smallmatrix} w_1 & 0 \\ w_2 & w_3 \\ 0 & w_4
	\end{smallmatrix}\right) = \left(\begin{smallmatrix}
	y_1 a & 0 \\ (y_2+y_3) a & 0 \\ 0 & y_4a
	\end{smallmatrix}\right)$ and $\left(\begin{smallmatrix}
	\sigma(w_5) & \sigma(w_6) & 0 \\ 0 & \sigma(w_7) & \sigma(w_8) \\
	\end{smallmatrix}\right) = \left(\begin{smallmatrix}
	\frac{y_2}{(y_2+y_3)a} & \frac{1}{a} & 0 \\ 0 & 0 & \frac{1}{a} \\
	\end{smallmatrix}\right)$, where $a \neq 0, 1/a \in \sigma(\mathbb{R})$, and using intermediate value theorem for continuous activation $\sigma(\cdot)$, we have $\frac{y_2}{(y_2+y_3)a} \in \sigma(\mathbb{R})$ because $\sigma(0)=0$ and $0<y_2<y_2+y_3$.
	Then we obtain $L(w_1, \dots, w_8) = \frac{y_1^2y_3^2}{2(y_2+y_3)^2} < y_1^2/2 <y_4^2/2$. Hence, a spurious valley exists. 
\end{proof}

\begin{remark}
	From Theorem \ref{thm:local-valley}, we already know the necessity of $n$ connections for each output neuron to guarantee no spurious valleys as dense networks. 
	Furthermore, we wonder whether a sparse network with at most $n{-}1$ connections for each output neuron and $n$ hidden neurons, is the same as $n{-}1$ hidden-neuron dense networks? 
	The answer is no. 
	We do allow at least $n$ effective neurons after pruning, while in narrow-network examples, no more than $ (n {-} 1) $ neurons exist. This leads to different expressive power. 
    We elaborate on this below. We denote $A_{a\times b}$ as a matrix in $\mathbb{R}^{a\times b}$. 
    For a dense network with $n{-}1$ hidden nodes, we have $\mathrm{rank} (U_{m \times (n-1)}\sigma(\widetilde{W}_{(n-1) \times d}X_{d\times n})) \leq n {-} 1$. In our result, we allow $m = n$ and a sparse matrix $\widetilde{U}_{ n \times n} = I_n$ 
    (i.e., only one connection for each output neuron), then from Theorem \ref{thm:non-linear}, $\mathrm{rank}(\widetilde{U}_{m \times n} \sigma(\widetilde{W}_{n \times d}X_{d\times n})) = n$, a.s. 
    Thus, wide sparse networks have more expressive power than dense narrow networks.
\end{remark}

From the proof of Theorem \ref{thm:local-valley}, we need to underline that such a spurious valley is a set of spurious minima generated from the sub-network when $w_4$ is removed. 
We encounter no spurious valleys when applied to wide SD networks (see Theorem \ref{thm:non-linear}). 
Thus, sparse connections indeed deteriorate the landscape because sparsity obstructs the decreasing path. 
Moreover, when the number of unpruned weights in each row of the final layer is less than the training sample size, a strict valley may indeed exist from our construction. 
Note that \citet{venturi2019spurious} also gave existence proof for spurious valleys on dense narrow networks.
The difference is that our example is concrete and applies to sparse structure with arbitrary width.
Furthermore, our example also covers the SS linear networks by adopting linear activation. 
Previous works have proven that there are no spurious valleys in dense linear networks \cite{venturi2019spurious}.
Hence, we conclude that sparsity in the final layer also deteriorates the landscape of linear networks.

\subsection{The instruction on reality}

Theorems \ref{thm:non-linear} and \ref{thm:deep-non-linear} show positive results for sparse networks. Specifically, with a dense final layer, the sparsity in other layers generally would not affect the benign landscape from the theoretical perspective. 
While Theorem \ref{thm:local-valley} provides the negative part, the sparsity in the final layer would bring spurious valleys.
These results convey the message that \textit{we need to be cautious when pruning the final hidden layer}, because sparsifying the final-layer weight could directly deteriorate the landscape of networks.
Intuitively, the final-layer weight decides the output of the network, which is crucial for fitting data from optimization even if the width is large enough. 
Thus, we need to preserve a dense final-layer weight if possible.  
Additionally, pruning the other layer weights has less impact on the landscape. 
Hence, we can adopt many variants of pruning methods in practice to obtain powerful features and produce high-efficient networks.
Nowadays, pruning methods adopt many tricks, such as iterative pruning and retraining methods \cite{liu2018rethinking, frankle2018lottery}. 
Hence, we recommend a small fraction of the pruning rates in the final-layer weight to avoid spurious valleys from the landscape perspective. 
Moreover, a recent work \cite[Appendix F]{liu2022the} also empirically discovers that allocating more parameters to the last fully-connected layer while keeping the overall parameter count fixed could lead to higher accuracy. Such experimental findings in \cite{liu2022the} demonstrate the importance of final layer weight from the generalization perspective, which support our viewpoint on the final-layer weight from another aspect.

\section{Experiments}\label{sec:exp}
Previous theoretical findings mainly describe the global landscape of sparse neural networks.
Now we conduct experiments to investigate deep sparse networks using practical optimization methods in this section.

\paragraph{Datasets}
We consider a synthetic dataset and the CIFAR10 dataset \citep{krizhevsky2009learning}. 
For the synthetic dataset, we sample training data $\bm{x}_i \stackrel{i.i.d.}{\sim} \mathcal{N}(\bm{0}, I_{100})$ and construct the corresponding target $\bm{y}_i = A\bm{x}_i + \bm{\varepsilon}_i$, where $\bm{\varepsilon}_i \stackrel{i.i.d.}{\sim} \mathcal{N}(\bm{0}, I_{d_y}), \bm{\varepsilon}_i \perp \!\!\! \perp \bm{x}_j$, and $A_{k l} \stackrel{i.i.d.}{\sim} \mathcal{N}(0, 1), \forall i, j \in [400], k \in [d_y], l\in[100]$. 
We scale $A$ to make $\|A\|_F = 5$ for avoiding a small numerical output when extreme sparsity is introduced. In the following, we construct two datasets with $d_y=1$ and $10$.
For the CIFAR10 dataset, we sample a subdataset with 400 training samples and use one-hot labels as targets ($d_y=10$).

\paragraph{Experimental Settings}
We use Mean Squared Error (MSE) loss with a unified constant learning rate of $0.01$ under the gradient descent (GD) method. And we construct the same hidden width of $400$ across all layers, which is the same as training samples to match our assumptions. 
We do not employ bias terms and any data pre-processing or data-augmentation in all of our experiments to monitor our theoretical schemes.
We use default initialization in PyTorch \cite{paszke2017automatic} among all experiments, unless specifically mentioned.

\begin{figure}[t]
	\centering
	\begin{subfigure}[b]{0.4\textwidth}
		\includegraphics[width=\linewidth]{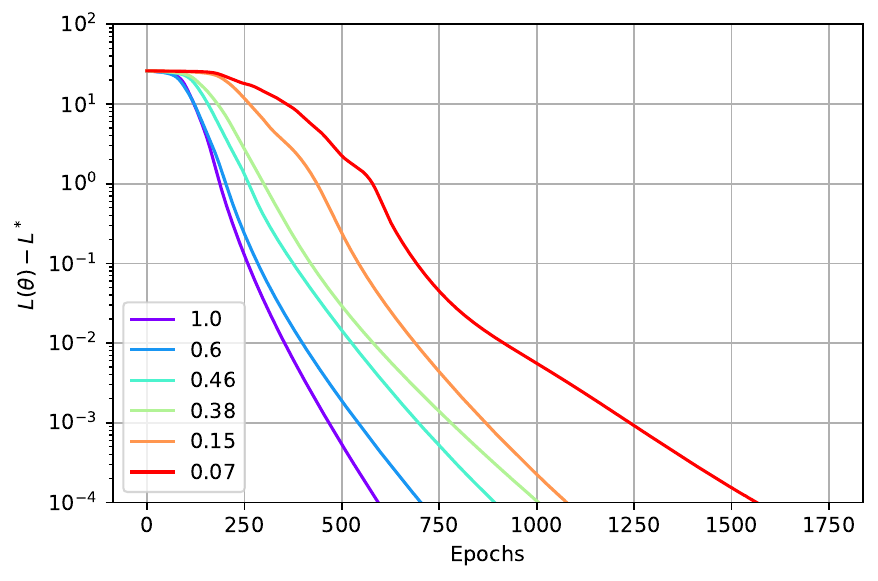}
		\caption{$d_y=10$.} \label{fig:deep-sparse-linear1b}
	\end{subfigure}
	\begin{subfigure}[b]{0.4\textwidth}
		\includegraphics[width=\linewidth]{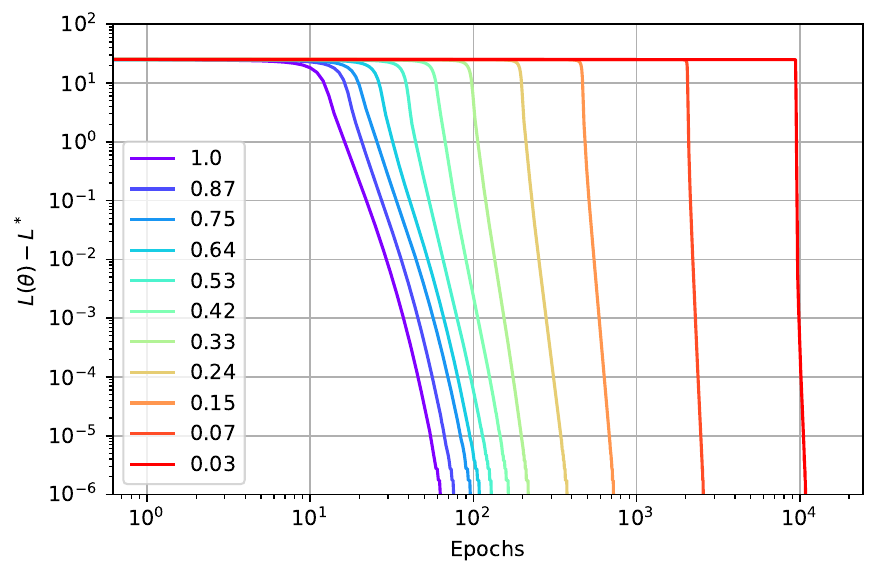}
		\caption{$d_y=1$.} \label{fig:deep-sparse-linear1a}
	\end{subfigure}
	\caption{Simulation of 5-layer linear sparse networks. (a) Only the first-layer weight is sparse and the sparse pattern is over-parameterized ($p_i\geq d_i, \forall i\in[s]$). The sparse ratio only applies to the first layer. (b) All layer weights are sparse, and total sparsity is shown in the legend. Legend `$1.0$' means the dense network. }
	\label{fig:deep-sparse-linear}
\end{figure}

\subsection{Sparse Linear Networks}
We first verify the results of deep sparse linear networks shown in Theorem \ref{thm:linear-deep}. Specifically, there is no spurious valleys for a deep sparse linear network when 1) only the first-layer weight is sparse and each sparse pattern in the sparse first layer is over-parameterized (i.e., $p_i\geq d_i, \forall i\in[s]$), which is shown in Figure \ref{fig:deep-sparse-linear1b} with $d_y=10$, or
2) all layer weights are sparse and the output dimension $d_y=1$, which is conducted in Figure \ref{fig:deep-sparse-linear1a}.
We adopt a $5$-layer linear network for example, because a large depth may cause the gradient vanishing and exploding \cite{sun2019optimization}.
To guarantee the \textit{effectiveness} of sparse networks, we first use random pruning with fixed pruning rates for each layer, and then check the useless connections, which may lead to a sparser structure. 

We plot the difference between the current loss ($L(\bm{\theta})$) and the global minimum ($L^*$) obtained from linear regression in Figure \ref{fig:deep-sparse-linear}. Here, $\bm{\theta}$ is the remaining parameters after random pruning and effectiveness detection, and from our constriction, $L^*=0$.
Though we only show no spurious valleys in Theorem \ref{thm:linear-deep}, we
suspect that a nice landscape enables the common optimizer to find a global minimum (i.e., zero loss) in practice.
The simulation verifies our conjecture: the optimization trajectory has nearly non-increasing loss and approaches to a global minimum.

\subsection{Sparse-dense Non-linear Networks}

\begin{figure}[t]
	\centering
	\begin{subfigure}[b]{0.3\textwidth}
		\includegraphics[width=\linewidth]{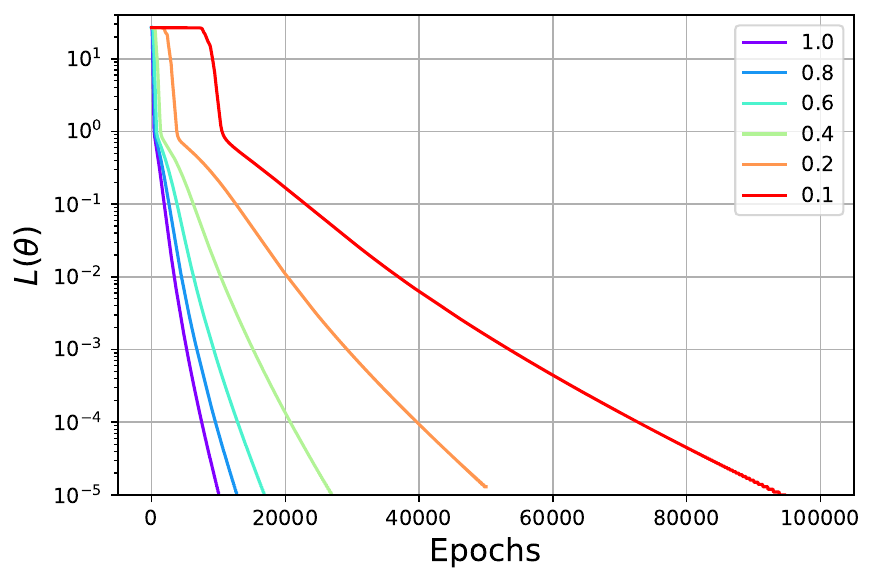}
		\includegraphics[width=\linewidth]{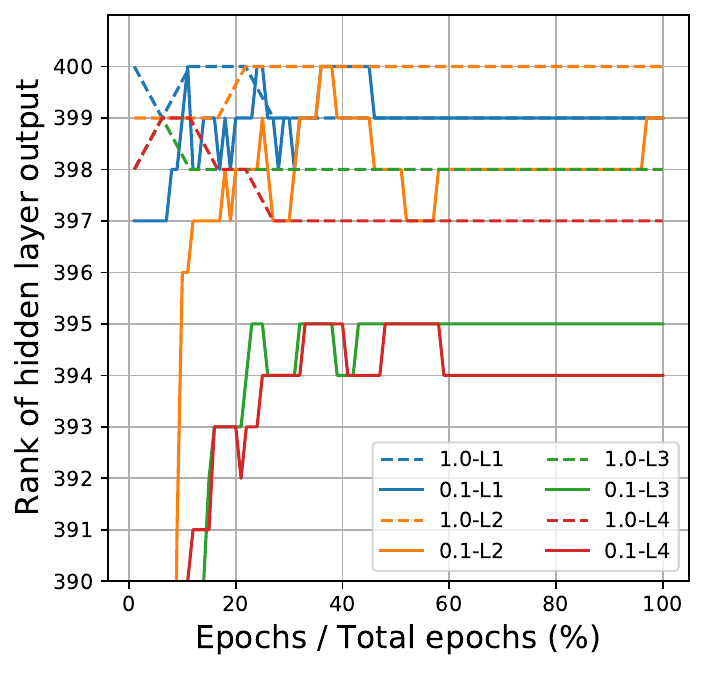}
		\caption{Tanh (Synthetic).} \label{fig:deep-sparse-nl-1a}
	\end{subfigure}
	\begin{subfigure}[b]{0.3\textwidth}
		\includegraphics[width=\linewidth]{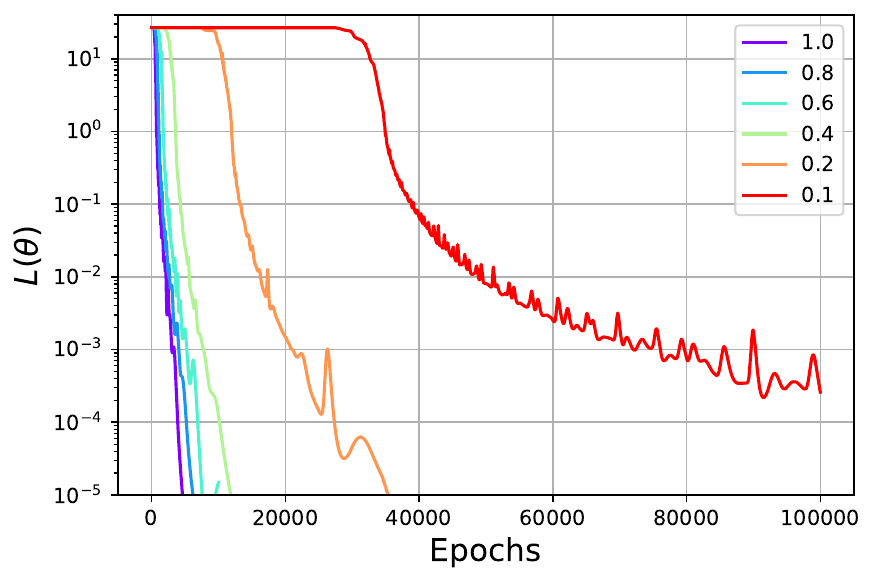}
		\includegraphics[width=\linewidth]{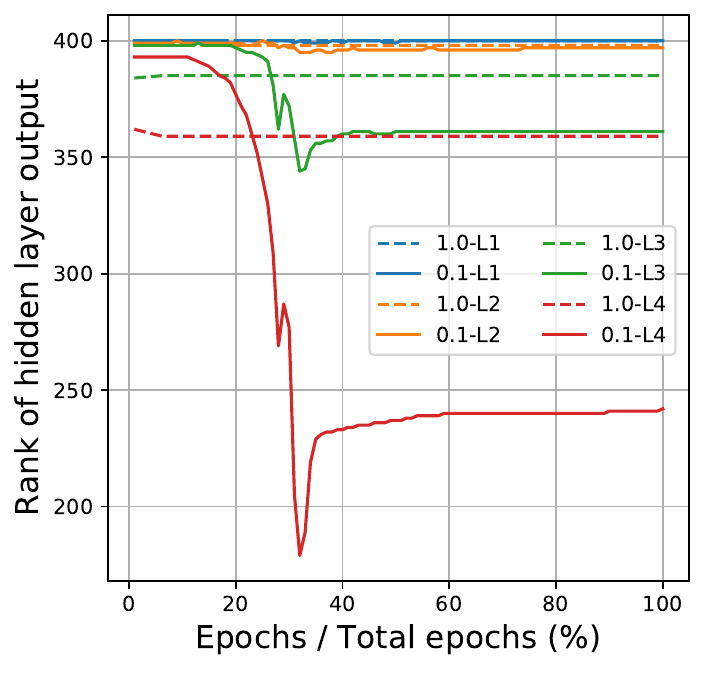}
		\caption{ReLU (Synthetic).} \label{fig:deep-sparse-nl-1b}
	\end{subfigure}	
	\begin{subfigure}[b]{0.3\textwidth}
	\includegraphics[width=\linewidth]{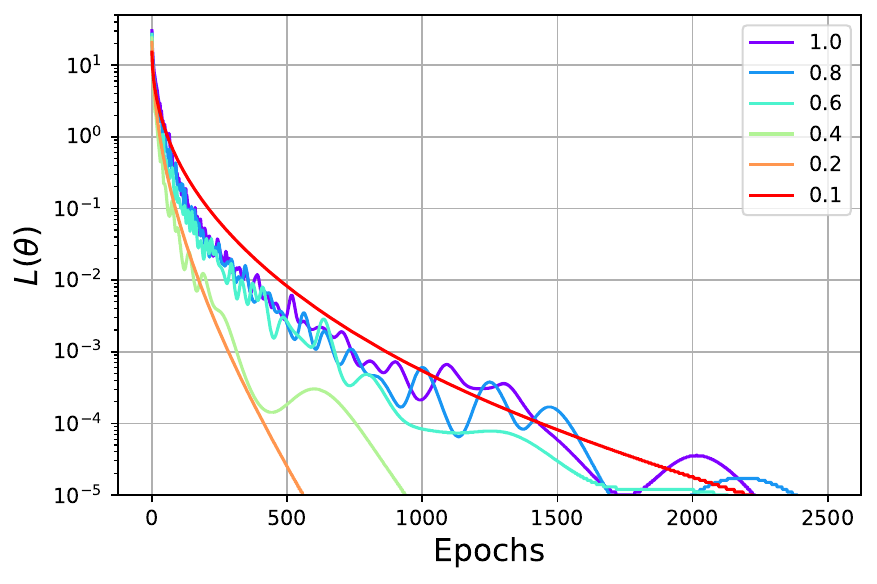}
	\includegraphics[width=\linewidth]{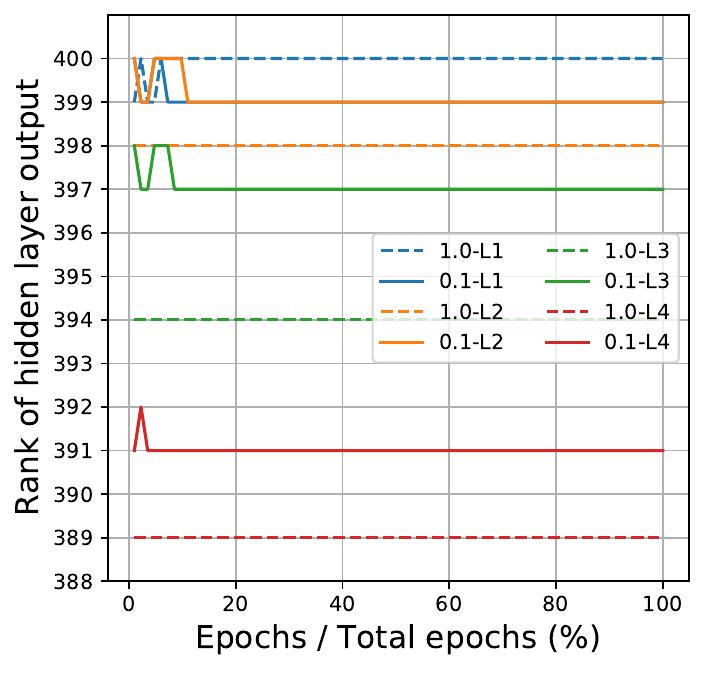}
	\caption{Tanh (CIFAR10).} \label{fig:deep-sparse-nl-1c}
	\end{subfigure}
	\caption{Simulation of 5-layer non-linear sparse networks with $d_y=10$ and total sparsity shown in the legend. We plot loss in the first row, and the rank of each hidden-layer output in the second row. Legend `0.1{-}L3' means the rank of $3$rd hidden-layer output under $0.1$ sparsity. 
	Since total running epochs are different for each sparsity, we rescale the epochs by total epochs to plot all results in one graph.}
	\label{fig:deep-sparse-non-linear}
\end{figure}

Second, we conduct experiments for deep sparse non-linear networks motivated by Theorem \ref{thm:deep-non-linear} to verify that the large width of the final layer in an SD net would preserve a benign landscape as the dense network, and the full rank property of each hidden-layer output matrix.
The numerical findings are shown in Figure \ref{fig:deep-sparse-non-linear}, where we use the dataset with $d_y=10$ and prune all layer weights except the final layer to guarantee $n$ connections for each output neuron.
For reasonable comparison, the input neurons all have at least a directed path to one final output. 
We will also examine whether each neuron in the last hidden layer is \textit{effective}.
During optimization, we discover that default initialization gives really slow convergence on the CIFAR10 dataset. Hence we multiply each weight by $10$ after initialization when running on the CIFAR10 dataset.

For non-linear networks in Figure \ref{fig:deep-sparse-non-linear}, we observe that all sparse structures could obtain near-zero training loss, which matches our theoretical finding. 
Moreover, we discover that the outputs of the first few hidden layers are full rank, 
but the output matrices of the later layers become rank deficient even for dense networks in Figure \ref{fig:deep-sparse-nl-1a}.
We suspect that such a phenomenon could be caused by numerical overflow from the saturation of activations and the small default initialization.
When applied to the CIFAR10 dataset with large initialization, we discover a stable rank phenomenon in Figure \ref{fig:deep-sparse-nl-1c}. 
Thus, we could confirm that the outputs of hidden layers are nearly full rank.

Additionally, we also attempt the non-analytic ReLU activation in Figure \ref{fig:deep-sparse-nl-1b}, though our theory does not cover it. 
We find it still has near-zero loss even with a large sparsity. But the training has a long distinct stuck period at the origin and oscillates around at last, and the rank variation under ReLU is also strange. 
For dense NNs, it preserves full rank property in the first few layers but becomes rank deficient at the last few hidden layers. 
When applied to sparse NNs, the rank drops quickly during training but does not recover at last.
We view such phenomenon as future work.

Overall, the approximate full rank property still holds when large sparsity is introduced as long as the effective width is larger than the number of training samples. 
But the sparsity usually introduces a much longer training trajectory employed with the default setting and has complex behavior coupled with optimization methods.

\subsection{Sparse-sparse Networks}\label{sec:ss-non}
\begin{figure}[t]
	\centering
	\begin{subfigure}[b]{0.4\textwidth}
		\includegraphics[width=1\linewidth]{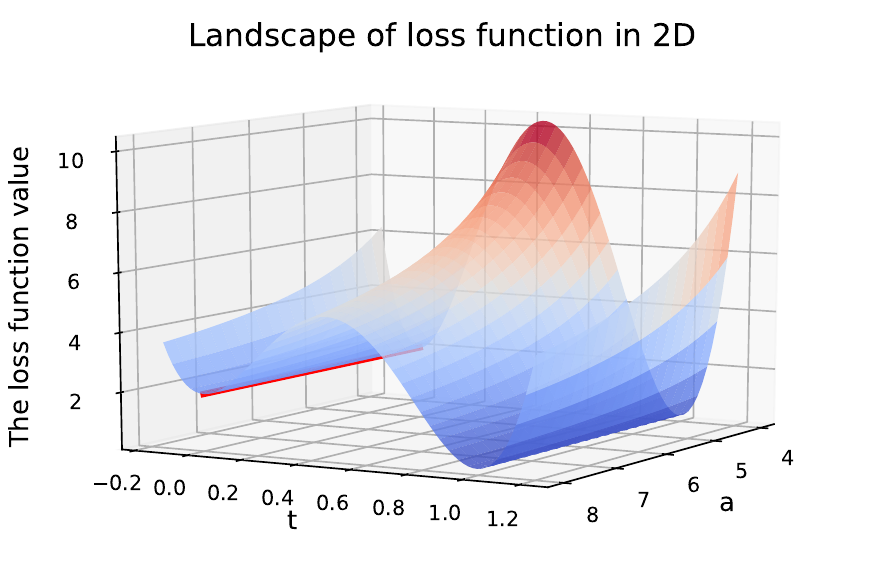}
		\caption{Tanh.} \label{fig:deep-sparse-nl-vis-1a}
	\end{subfigure}
	\begin{subfigure}[b]{0.4\textwidth}
		\includegraphics[width=1\linewidth]{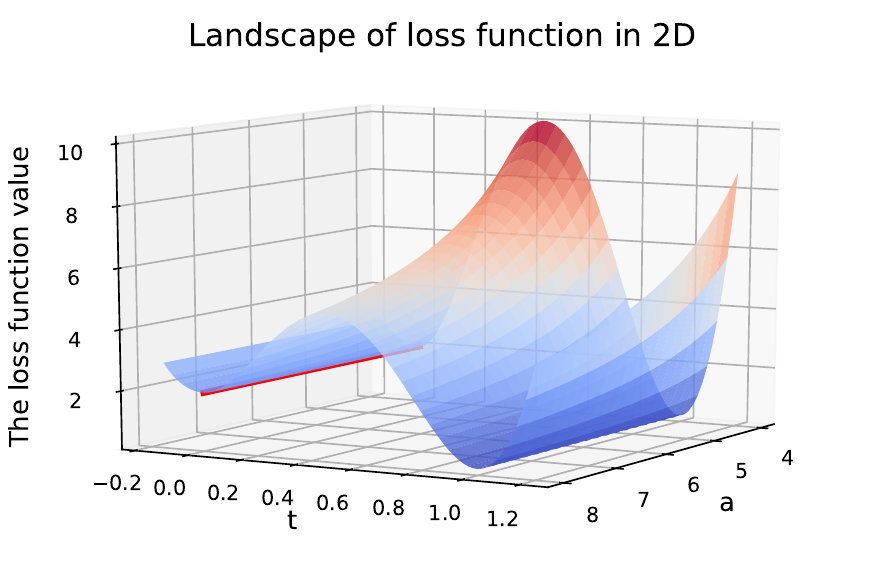}
		\caption{LeakyReLU.} \label{fig:deep-sparse-nl-vis-1b}
	\end{subfigure}
	\caption{Loss landscape visualization for the examples in Section \ref{sec:ss-non}. We project the landscape into a 2D space spanned by $a$ and $t$, where $t$ describes the interpolation: $t\tilde{\bm{\theta}}(a)+(1-t)\bm{\theta}(a, b)$.
	Here $a$, $\tilde{\bm{\theta}}(a)$ and $\bm{\theta}(a,b)$ follow the proof of Theorem \ref{thm:local-valley}.
	The red line is the loss value at $\bm{\theta}(a,b)$ with $b=10$ and various $a>0$, which form a spurious valley.}
	\label{fig:deep-sparse-nonlinear-vis}
\end{figure}

Finally, we verify the bad landscape for SS nets constructed in the proof of Theorem \ref{thm:local-valley}. 
For simplification, we select target $X=I_3, Y = \left(\begin{smallmatrix}
1 & 1 & 0 \\
2 & 8 & 0 \\
0 & 0 & 2 \\
\end{smallmatrix}\right)$\footnote{Indeed, we only need $y_3>y_4>y_1>0, y_2>0$ in the proof in Theorem~\ref{thm:local-valley}.} following Eq.~\eqref{eq:cx-obj} and consider the sparse structure in the right graph of Figure \ref{fig:1} with $d_x=3, p=2, d_y=3$, i.e., the loss function is 
\begin{equation*}
    \begin{aligned}
        L(\bm{w}) &= \frac{1}{2}\left\| \left(\begin{smallmatrix}
    	w_1 & 0 \\ w_2 & w_3 \\ 0 & w_4 
    	\end{smallmatrix}\right)\sigma\left[
    	\left(\begin{smallmatrix}
    	w_5 & w_6 & 0 \\ 0 & w_7 & w_8
    	\end{smallmatrix}\right)X\right]-Y
    	\right\|_F^2 = \frac{1}{2}\left\|
    	\left(\begin{smallmatrix}
    	w_1\sigma(w_5)-1 & w_1\sigma(w_6)-1 &  0 \\ 
    	w_2\sigma(w_5)-2 & w_2\sigma(w_6)+w_3\sigma(w_7)-8 & w_3\sigma(w_8) \\ 
    	0 & w_4\sigma(w_7) & w_4\sigma(w_8)-2
    	\end{smallmatrix}\right)\right\|_F^2.
    \end{aligned}
\end{equation*}

We use many common continuous activations including shifted Sigmoid \footnote{We choose the shifted Sigmoid activation as $\sigma(z) = 1/(1+e^{-z})-1/2$ to satisfy $\sigma(0)=0$.}, Tanh, LeakyReLU, ELU \cite{clevert2015fast} and ReLU, which all satisfy our activation assumption of $\sigma(0)=0$. 
We totally run $100$ trials using GD with a learning rate of $0.01$ until the training converges (up to $50000$ iterations) and record the distribution of convergent solutions in Table \ref{table:res3}. 

From Table \ref{table:res3}, except for ReLU, the training loss for other activations either approaches a global minimum, or traps into the valley we constructed. 
And we could observe that it is likely to fall into a spurious valley instead of a global minimum. 
Moreover, we discover that for ReLU activation, there exist $11$ final convergent solutions while only $5$ successful times approach global minima, showing a worse landscape compared to other activations.
Finally, the landscape visualization of some activations in Figure \ref{fig:deep-sparse-nonlinear-vis} also matches well with our theorem.
Overall, we could encounter severe landscape issues in SS nets from our theorems and experiments.

\begin{savenotes}
	\begin{table}[t]
		\renewcommand{\arraystretch}{0.8}
		\centering
		\begin{tabular}{|c|c|c|c|c|}
			\hline
			Shifted Sigmoid & Tanh & LeakyReLU & ELU & ReLU \\
			\hline
			$ \fbox{95},\mathbf{5} $ & $ \fbox{99},\mathbf{1} $ & $ \fbox{45},\mathbf{55} $ & $ \fbox{67},\mathbf{33} $ & \makecell{$17,15,\mathbf{5},19,6,$\\$6,6,8,\fbox{7},7,4$} \\
			\hline
		\end{tabular}
		\caption{Statistic of final convergent solutions during totally $100$ trials for different continuous activations. The number inside a square is the times trapped in our constructed valleys and the bold number is the times that the algorithm generates an approximate global minimum. Particularly, we discover totally 11 different solutions under ReLU.}
		\label{table:res3}
	\end{table}
\end{savenotes}

\section{Conclusion}\label{sec:conclusion}
We have investigated the landscape of sparse networks for several scenarios. 
For a sparse linear network with a dense final layer, we have shown that under some specific structures there are no spurious valleys.
For a sparse non-linear network with a dense wide final layer, there is no difference with the dense case in terms of spurious valleys.
Finally, we have shown that spurious valleys can exist for networks with both sparse layers.
These results  do reveal the negative impact of sparsity on the loss landscape of neural nets. 
Future research directions include providing more necessary conditions for sparse non-linear networks to have a benign landscape, and analyzing the optimization trajectory under popular gradient descent methods when applied to sparse networks.

\appendix

\section{Useless Connections and Nodes in Sparse Networks}\label{app:useless-connect}
We show several kinds of useless connections caused by sparsity or network pruning.

\begin{enumerate}
	\item Zero out-degree \uppercase\expandafter{\romannumeral1}: if a node has zero out-degree, we can eliminate its input connections because this node does not influence the network's output, e.g., $h^{(2)}_1$ in Figure \ref{fig:useless-node}.
	\item Zero out-degree \uppercase\expandafter{\romannumeral2}: if a node has zero out-degree after removing its output connections based on the criterion of Zero out-degree \uppercase\expandafter{\romannumeral1} in latter layers, we also can eliminate its input connections,
	e.g., $h^{(1)}_1$ in Figure \ref{fig:useless-node}. Though it owes one output connection, the connected node $h^{(2)}_1$ has zero out-degree. Hence, the connection can be removed, leading to zero out-degree. So we can eliminate the input connections of $h^{(1)}_1$ as well.
	\item Zero in-degree \uppercase\expandafter{\romannumeral1}: if a node has zero in-degree, we can eliminate its output connections when $\sigma(0)=0$ and no bias term is used, because this node only provides zero feature, e.g., $h^{(2)}_4$ and $h^{(1)}_4$ in Figure \ref{fig:useless-node}. However,  when the node has a bias term, then we can not remove output connections since the bias constant will propagate to subsequent layers, though such a node only provides the feature like $\sigma(b\cdot \bm{1}_n)$ for some $b \in \mathbb{R}$.
	\item Zero in-degree \uppercase\expandafter{\romannumeral2}: if a node has zero in-degree after removing its input connections based on the criterion of Zero in-degree \uppercase\expandafter{\romannumeral1} in former layers, we also can eliminate its output connections, e.g., $h^{(2)}_3$ in Figure \ref{fig:useless-node}. 
	Though it owes one input connection, the connected node $h^{(1)}_4$ has zero in-degree. Hence the connection can be removed, leading to zero in-degree. So we can eliminate the output connections of $h^{(2)}_3$ as well. Similarly, we can not remove output connections if the bias term is employed.
\end{enumerate}

In conclusion, we can remove hidden nodes that do not appear in any directed path from one input node to one output node. Then the remaining structure is \textit{effective}.

\def\layersep{1.5cm}
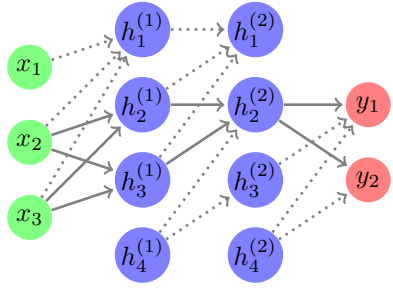
\begin{figure}[t]
	\centering
	\begin{tikzpicture}[shorten >=1pt,->,draw=black!50, node distance=\layersep]
	
	\tikzstyle{every pin edge}=[<-,shorten <=1pt]
	\tikzstyle{neuron}=[circle,fill=black!25,minimum size=17pt,inner sep=0pt]
	\tikzstyle{input neuron}=[neuron, fill=green!50];
	\tikzstyle{output neuron}=[neuron, fill=red!50];
	\tikzstyle{hidden neuron}=[neuron, fill=blue!50];
	\tikzstyle{annot} = [text width=4em, text centered]
	
	% Draw the input layer nodes
	\foreach \name / \y in {1,...,3}
	% This is the same as writing \foreach \name / \y in {1/1,2/2,3/3,4/4}
	\node[input neuron] (I-\name) at (0,-\y) {$x_{\y}$};
	
	% Draw the hidden layer nodes
	\foreach \name / \y in {1,...,4}
	\path[yshift=0.5cm]
	node[hidden neuron] (H1-\name) at (\layersep,-\y cm) {$h^{(1)}_{\y}$};
	
	% Draw the hidden layer nodes
	\foreach \name / \y in {1,...,4}
	\path[yshift=0.5cm]
	node[hidden neuron] (H2-\name) at (2*\layersep,-\y cm) {$h^{(2)}_{\y}$};
	
	% Draw the output layer node
	\foreach \name / \y in {1,2}
	\path[yshift=-0.5cm]
	node[output neuron] (O-\name) at (3*\layersep,-\y cm) {$y_{\y}$};
	
	% Connect every node in the input layer with every node in the
	% hidden layer.
	\foreach \source in {2,3}
	\foreach \dest in {2,3}
	\path (I-\source) edge [line width=1pt](H1-\dest);
	
	\foreach \source in {1,2,3}
	\foreach \dest in {1}
	\path (I-\source) edge [dotted, line width=1pt](H1-\dest);
	
	\foreach \source in {1,2,3}
	\foreach \dest in {1}
	\path (H1-\source) edge [left, dotted, line width=1pt, pos=0.7, above, align=center] node {} (H2-\dest);
	
	\foreach \source in {2,3}
	\foreach \dest in {2}
	\path (H1-\source) edge [line width=1pt] (H2-\dest);	
	
	\foreach \source in {4}
	\foreach \dest in {2,3}
	\path (H1-\source) edge [dotted, line width=1pt] (H2-\dest);		
	
	% Connect every node in the hidden layer with the output layer
	\foreach \source in {2}
	\foreach \dest in {1, 2}
	\path (H2-\source) edge [left, line width=1pt] (O-\dest);
	
	\foreach \source in {3}
	\foreach \dest in {1}
	\path (H2-\source) edge [left, dotted, line width=1pt] (O-\dest);
	
	\foreach \source in {4}
	\foreach \dest in {1,2}
	\path (H2-\source) edge [dotted, line width=1pt] (O-\dest);
	
	% Annotate the layers
	\node[annot,above of=H1-1, node distance=1cm] (hl1) {Hidden 1};
	\node[annot,above of=H2-1, node distance=1cm] (hl2) {Hidden 2};
	\node[annot,left of=hl1] {Input};
	\node[annot,right of=hl2] {Output};
	\end{tikzpicture}
	\caption{An example of sparse networks. Lines are connections of the original sparse network, dotted lines are useless connections that can be removed, and solid lines are effective connections.} \label{fig:useless-node}
	\vspace{-10pt}
\end{figure}

\section{Missing Proofs for Section \ref{section:general-linear}} \label{appendix:general-linear}

\subsection{Proof of Theorem \ref{thm:two-layer-linear}} \label{app:shallow-linear}
\begin{lemma}[Modified from \cite{venturi2019spurious}]\label{lemma:zero-path}
	For two matrices $U = [\bm{u}_1, \ldots, \bm{u}_m] \in \mathbb{R}^{m \times p}$ and $W = [\bm{w}_1, \ldots, \bm{w}_p]^\top \in \mathbb{R}^{p \times n}$, $p\geq n$. If $\mathrm{rank}(W)=r < n$, and $\bm{w}_1, \ldots, \bm{w}_r$ are linearly independent, then we can construct a $U^0$, such that $U^0W = U W$ and $U^0$ has $(p{-}r)$ zero columns.
\end{lemma}
\begin{proof}
	From the condition, 
	$ U = ( \strut\smash{\stackrel{r}{U_1}} \ 
	\strut\smash{\stackrel{p-r}{U_2}} ), W = (
	\strut\smash{\stackrel{r}{W_1}} \ 
	\strut\smash{\stackrel{p-r}{W_2}} )^\top = \left(\begin{smallmatrix}
	I_r & \bm{0} \\
	Q & \bm{0} \\
	\end{smallmatrix}\right) \left(\begin{smallmatrix}
	W_1^\top \\
	\bm{0}
	\end{smallmatrix}\right)$, where $Q \in \mathbb{R}^{(p-r) \times r}$ because $W_1^\top$ is a basis of the row space of $W$.
	Then
	$ U W = \left[ (U_1 \ U_2) \left(\begin{smallmatrix}
	I_r & \bm{0} \\
	Q & \bm{0} \\
	\end{smallmatrix}\right)\right] \left(\begin{smallmatrix}
	W_1^\top \\
	\bm{0} \\
	\end{smallmatrix}\right) = (U_1+U_2Q \ \ \bm{0})\left(\begin{smallmatrix}
	W_1^\top \\
	\bm{0} \\
	\end{smallmatrix}\right) = (U_1+U_2Q \ \ \bm{0})W$.
	Therefore, we can choose $U^0 = (U_1+U_2Q \ \ \bm{0})$, which satisfies the requirement.
\end{proof}

\begin{lemma}\label{lemma:convex}
	Suppose $f(\bm{u})\colon \mathbb{R}^d \to \mathbb{R}$ is a convex function of $\bm{u}$, and $\bm{u}^* \in\mathop{\arg\min}_{\bm{u}}f(\bm{u})$. Then given any $\bm{u}_0\in\mathbb{R}^d$, we have $f((1-t)\bm{u}_0 + t \bm{u}^*)$ as a function of $t \in [0,1]$ is non-increasing.
\end{lemma}
\begin{proof}
    For any $1 \geq t > s \geq 0$, by the convexity of $f(\cdot)$ and $f(\bm{u}_*)=\min_{\bm{u}}f(\bm{u})$, we have 
    \[ f((1-s) \bm{u}_0 + s \bm{u}^*) \geq \frac{1-t}{1-s}\cdot f((1-s) \bm{u}_0 + s \bm{u}^*) + \frac{t-s}{1-s}\cdot f(\bm{u}^*) \geq f((1-t)\bm{u}_0 + t \bm{u}^*). \]
    Thus, $f((1-t)\bm{u}_0 + t \bm{u}^*)$ as a function of $t \in [0,1]$ is non-increasing.
\end{proof}

Now we turn back to the proof of Theorem \ref{thm:two-layer-linear}.

\begin{proof}
	We prove the results of no spurious valleys by constructing the path that satisfies {\PP}.
	\begin{enumerate}
		\item[1)] If $p_i\geq d_i$, we reformulate the objective as 
		\begin{equation*}
    		\begin{aligned}
    		    &\min_{\substack{U_1, \dots, U_s, \\ W_1, \dots, W_s}} L(U_1, \dots, U_s, W_1, \dots, W_s) := \frac{1}{2}\left\| (U_1 \ \cdots \ U_s) \left(\begin{smallmatrix}
    		W_1    & \dots  & \bm{0} \\
    		\vdots & \ddots & \vdots \\
    		\bm{0} & \dots  &    W_s \\
    		\end{smallmatrix}\right)Z-Y\right\|_F^2, Z := \left(\begin{smallmatrix}
    		Z_1 \\ \vdots \\ Z_s
    		\end{smallmatrix}\right)
    		\end{aligned}
		\end{equation*}
		Now we construct a non-increasing path to a global minimum. If for certain $i$, $\mathrm{rank}(W_i)=r<d_i$, we denote $W_i=[\bm{w}_1,\dots,\bm{w}_{p_i}]^\top$, and w.l.o.g., $\bm{w}_1\dots, \bm{w}_r$ is linearly independent.
		Then from Lemma \ref{lemma:zero-path}, we are able to find $U_i^0$, such that $U_i^0W_i = U_i W_i$, and the last $p_i-r$ columns of $U_i^0$ are zero. We first fix $W_i$ and take a linear path $U_i(t) = (1-t)U_i+t U_i^0, 0 \leq t \leq 1$ with invariant loss. 
		Then we fix $U_i^0$ and take another linear path $W_i(t')=(1-t')W_i+t' W_i^0, 0 \leq t' \leq 1$, where $W_i^0$ has the same $r$ rows as $W_i$ and $\mathrm{rank}(W_i^0)=d_i$ (we can always extend linearly independent vectors to a basis due to $p_i \geq d_i$).
		Hence, we could reach a point $(U_1^0, \dots, U_s^0, W_1^0, \dots, W_s^0)$ with all full column rank $W_i^0$s with unchanged loss.

		Next, note that a global minimizer $U_i^*, W^*_i$s satisfies $\left(U_1^*W_1^*, \dots, U_s^* W_s^*\right)=Y Z^+$, where $Z^+$ is the pseudoinverse of $Z$. Hence we could obtain a global minimizer by making $U_i^*$s satisfy $\left(U_1^*W^0_1, \dots, U_s^* W_s^0\right)=Y Z^+$ since $W_i$s have full column rank.
		Finally, we take $U_i(t'') = (1-t'')U_i^0 + t'' U_i^*, 0 \leq t'' \leq 1, \forall i \in[s]$. Since $L(U_i, W_i, i\in[s])$ is convex related to $U_i$s, the constructed path is a non-increasing path to the global minimum by Lemma \ref{lemma:convex}.
		Hence, there are no spurious valleys.
		\item[2)] If $Z_i Z_j^\top=\bm{0}, \forall i,j \in[s], i\neq j$, then $Z_i$ and $Z_{j}$ share no same rows of $X$. Additionally, each row in data $X$ has connections to the first layer weight $W$, thus $Z_{1}, \ldots, Z_{s}$ are an arrangement of the rows of $X$. Then we only need to consider the objective:
		\begin{equation*}
		\begin{aligned}
		\min_{U, W_1, \dots, W_s} L(U, W_1, \dots, W_s) &= \frac{1}{2}\left\|\sum_{i=1}^s U_i W_i Z_i -Y\right\|_F^2 \stackrel{(i)}{=} \frac{1}{2}\sum_{i=1}^{s} \left\|U_i W_i Z_i-Y\right\|_F^2-\frac{s-1}{2} \left\|Y\right\|_F^2,
		\end{aligned}
		\end{equation*}
		where $(i)$ uses the assumption $Z_i Z_j^\top=\bm{0}, \forall i,j \in[s], i\neq j$ again.
		We see that the objective has already been separated into $s$ parts, while each part is a two-layer dense linear network. Based on \cite[Theorem 11]{venturi2019spurious}, dense linear networks have no spurious valleys. Hence the original SD network has no spurious valleys.
		
		\item[3)] If $d_y=1$, we can simplify $U=\left(u_1,\dots, u_p\right)\in\mathbb{R}^{1\times p}$, and with some abuse of notation, we set the rows of $W_1, \dots W_s$ in sequence as $\bm{w}_1^\top, \dots, \bm{w}_p^\top$ ($\bm{w}_i$s may not have same dimensions) with corresponding data $Z_1, \dots, Z_p$ (some of $Z_i$s could be same), then we can reformulate the original problem as
		\begin{equation*}\label{eq: pruned-problem}
		\min_{u_i, \bm{w}_i, i\in[p]} L(U, W) = \frac{1}{2}\left\| \sum_{i=1}^{p} u_i \bm{w}_i^\top Z_i-Y \right\|_2^2.
		\end{equation*}
		From any initialization, if certain $u_i=0, i\in[p]$, we can construct linear paths from $\bm{w}_i$ to $\bm{w}_i'=\bm{0}$, and then from $u_i=0$ to $u_i'=1$ with fixed other weights, where the loss is invariant. Therefore, we can assume $u_i \neq 0, \forall i\in[p]$. Then the objective is convex related to the sparse first-layer weight $\widetilde{W}$. Hence, we can construct a non-decreasing path to a global minimum by Lemma \ref{lemma:convex}. 
% 		Moreover, the loss is a global minimum of the dense network, since the network is \textit{effective} and the second-layer weights are all non-zero.
	\end{enumerate}
\end{proof}

\subsection{Proof of Theorem \ref{thm:linear-deep}} \label{app:ext-linear}
\begin{proof}
    The objective of an $L + 1$-layer network is $ \min_{V_i, i\in[L], W} \frac{1}{2}\left\|V_L\cdots V_1 W X - Y\right\|_F^2$. We first focus on the sparse first layer weight (i.e., $W$ is sparse) as shallow linear networks.
    We prove the results of no spurious valleys by constructing the path that satisfies {\PP} following \citet[Lemma 2]{venturi2019spurious}.
	\begin{enumerate}
		\item[1)] If $p_i\geq d_i$, we reformulate the objective as $(P)$ below:
		\begin{equation*}
		\begin{aligned}
		\min_{\substack{W_i, i\in[s] \\ V_j, j\in[L]}} L(V_L, \dots, V_1, W_1,\dots, W_s) 
		:= \frac{1}{2}\left\|
		V_L\cdots V_1\left(\begin{smallmatrix}
		W_1    & \dots  & \bm{0} \\
		\vdots & \ddots & \vdots \\
		\bm{0} & \dots  &    W_s \\
		\end{smallmatrix}\right)\left(\begin{smallmatrix}
		Z_1 \\ \vdots \\ Z_s
		\end{smallmatrix}\right)-Y\right\|_F^2.
		\end{aligned}
		\end{equation*}
		Invoked from the proof of 1) in Theorem \ref{thm:two-layer-linear}, we can assume $W_i$s have full column rank, then we can take 
		$\widetilde{V}_1 := V_1\left(\begin{smallmatrix}
		W_1    & \dots  & \bm{0} \\
		\vdots & \ddots & \vdots \\
		\bm{0} & \dots  &    W_s \\
		\end{smallmatrix}\right)$ as new variable. Note that choosing $V_1=V\left(\begin{smallmatrix}
		W_1^+  & \dots  & \bm{0} \\
		\vdots & \ddots & \vdots \\
		\bm{0} & \dots  &  W_s^+ \\
		\end{smallmatrix}\right)$
		for any matrix $V$. We can see there has no constraint for $\widetilde{V}_1$. 
		Thus, the original problem shares the same global minimum of $(P_1)$ below:
		\begin{equation*}
		\begin{aligned}
		\min_{V_L, \dots,V_2,\widetilde{V}_1} \widetilde{L}(V_L\dots,V_2,\widetilde{V}_1) := \frac{1}{2} \left\|	V_L\cdots V_2\widetilde{V}_1\left(\begin{smallmatrix}
		Z_1 \\ \vdots \\ Z_s
		\end{smallmatrix}\right)-Y\right\|_F^2.
		\end{aligned}
		\end{equation*}
		From \citet[Theorem 11]{venturi2019spurious}, we can construct a non-decreasing path to a global minimum of $(P_1)$ from any initialization. Thus we could obtain a non-decreasing path to a global minimum for $(P)$ as well.
		
		\item[2)] If $d_y=1$, we directly consider the objective with all sparse layers as 
		\begin{equation*}
		\min_{\substack{\bm{w}_i, i\in[p] \\ \widetilde{V}_j, j\in[L]}} 
		\frac{1}{2}\left\|\widetilde{V}_L\cdots \widetilde{V}_1\left(\begin{smallmatrix}
		\bm{w}_1^\top Z_1 \\
		\vdots \\
		\bm{w}_p^\top Z_p \\
		\end{smallmatrix}\right)-Y\right\|_F^2,
		\end{equation*}
		where $\widetilde{V}_i \in \mathbb{R}^{p_i\times p_{i-1}}, \forall i \in [L]$ are sparse layer wights and $p_0=p$, $\bm{w}_i$ is the remaining parameters in the $i$-th row of the first layer weight, and $Z_i$ is the corresponding data matrix. Note that $(\widetilde{V}_L\cdots \widetilde{V}_1) \in \mathbb{R}^{1 \times p}$ due to $d_y=1$.
		If the initialization satisfies $(\widetilde{V}_L\cdots \widetilde{V}_1)_j \neq 0, \forall j \in[p]$, then the objective is convex related to $\bm{w}_i, i\in[p]$. 
		Thus we can construct a non-increasing path to its global minimum by Lemma \ref{lemma:convex}. Otherwise, for certain $j \in [p]$, such that $(\widetilde{V}_L\cdots \widetilde{V}_1)_j = 0$, we fix $j$ and prove by induction for $L$ that we can construct a path with the invariant loss to a point $(\widetilde{V}_L', \dots, \widetilde{V}_1', \bm{w}_1, \dots, \bm{w}_j', \dots, \bm{w}_p)$ such that $\bm{w}_j'= \bm{0}$ and $\widetilde{V}_L'\cdots \widetilde{V}_1'= \widetilde{V}_L\cdots \widetilde{V}_1 + a\bm{e}_j^\top$ for some $a\neq 0$, i.e., $\widetilde{V}_L'\cdots \widetilde{V}_1'$ only change the $j$-th dimension of $\widetilde{V}_L\cdots \widetilde{V}_1$ to be nonzero.
		The case $L = 1$ has shown in shallow linear network. 
		Now suppose the result hold for $L - 1$ layers, and we consider the case of $L$ layers. 
		
		\textbf{Case 1.} Suppose for some $i\in[p_1]$, $(\widetilde{V}_L\cdots \widetilde{V}_2)_i \neq 0$ and $(\widetilde{V}_1)_{i,j}$
		is not pruned. We first construct a linear path from $\bm{w}_j$ to $\bm{w}_j'=\bm{0}$ with fixed other weights, which gives the invariant loss due to the hypothesis that $(\widetilde{V}_L\cdots \widetilde{V}_1)_j = 0$. Then we can construct another linear path starting from $\widetilde{V}_1$ to $\widetilde{V}_1'$ with $(\widetilde{V}_1')_{\cdot, j} = \bm{e}_i$ and the same other columns.
		The loss is still unchanged because $\bm{w}_j'=\bm{0}$.
		Now we obtain $(\widetilde{V}_L\cdots \widetilde{V}_2\widetilde{V}_1')_j = (\widetilde{V}_L\cdots \widetilde{V}_2) \cdot (\widetilde{V}_1')_{\cdot, j} = (\widetilde{V}_L\cdots \widetilde{V}_2)_i \neq 0$, and $(\widetilde{V}_L\cdots \widetilde{V}_2\widetilde{V}_1')_k = (\widetilde{V}_L\cdots \widetilde{V}_1)_k, \forall k \neq j$, which satisfies our requirement.
		
		\textbf{Case 2.} Otherwise, for all $i \in [p_1]$, such that $(\widetilde{V}_1)_{i,j}$
		is not pruned, we have $(\widetilde{V}_L\cdots \widetilde{V}_2)_i = 0$.
		Now we view $X_{new} := \left(\begin{smallmatrix}
		\bm{w}_1^\top Z_1 \\
		\vdots \\
		\bm{w}_p^\top Z_p \\
		\end{smallmatrix}\right)$, and rearrange the data to define $\left(Z_{new}\right)_{t}, t\in[p_1]$ similarly as Eq.~\eqref{eq:objective}. 
		Then we have
		\[ \widetilde{V}_L \cdots \widetilde{V}_1X_{new} = 
		\widetilde{V}_L \cdots \widetilde{V}_2 \left(\begin{smallmatrix}
		\left(\widetilde{V}_1 \right)_{1,\cdot} \left(Z_{new}\right)_{1} \\
		\vdots \\
		\left(\widetilde{V}_1 \right)_{p_1,\cdot} \left(Z_{new}\right)_{p_1} \\
		\end{smallmatrix}\right). \]
	    For each $i$ that satisfies $(\widetilde{V}_L\cdots \widetilde{V}_2)_i = 0$, from the inductive hypothesis of $L-1$ layers, we can construct a path to a point $(\widetilde{V}_L', \dots, \widetilde{V}_2', (\widetilde{V}_1)_{1,\cdot}^\top, \dots, (\widetilde{V}_1')_{i,\cdot}^\top, \dots, (\widetilde{V}_1)_{p_1, \cdot}^\top)$ that only changes $i$-th dimension of $\widetilde{V}_L\cdots \widetilde{V}_2$ to be nonzero and makes $(\widetilde{V}_1')_{i, \cdot}=\bm{0}$ with the invariant loss. 
		Thus, $(\widetilde{V}_1')_{i, j}=0$.
		Repeating the above argument for all $i$, we obtain $(\widetilde{V}_1')_{\cdot, j}=\bm{0}$. Finally, we construct a loss-invariant linear path from $\bm{w}_j$ to $\bm{w}_j'=\bm{0}$ due to $(\widetilde{V}_1')_{\cdot, j}=\bm{0}$. Subsequently, noting that $\bm{w}_j'=\bm{0}$, we can construct another loss-invariant linear path from $(\widetilde{V}_1')_{\cdot, j}$ to $(\widetilde{V}_1'')_{\cdot, j}=\bm{e}_k$ by choosing a $k \in [p_1]$ that $(\widetilde{V}_1)_{k,j}$ is not pruned and $(\widetilde{V}_L'\cdots \widetilde{V}_2')_k \neq 0$, which completes our inductive step because $(\widetilde{V}_L'\cdots \widetilde{V}_2'\widetilde{V}_1'')_j = (\widetilde{V}_L'\cdots \widetilde{V}_2')_k \neq 0$.
		
		Combining \textbf{Case 1} and \textbf{Case 2}, we finish our induction. Therefore, we can always reach a point with $(\widetilde{V}_L\cdots \widetilde{V}_1)_j\neq 0, \forall j \in[p]$, and construct a non-increasing path with only adjusting $\bm{w}_i$s to a global minimum by Lemma \ref{lemma:convex}.
	\end{enumerate}
\end{proof}

\section{Missing Proofs for Section \ref{sec:non-linear}} \label{app:non-linear}
\subsection{Proof of Theorem \ref{thm:infinite-dim-non-linear}}

\begin{proof}
    The proof is partially inspired by \citet{venturi2019spurious}, but we extend the result to sparse networks.
	Since $\dim^*(\sigma, X) < \infty $, then $\forall i\in[s]$, $\dim^*(\sigma, Z_i) < \infty $. Thus $\forall \bm{w}, \bm{z}\in\mathbb{R}^{d_i}$, we have $\sigma(\bm{w}^\top\bm{z})=\langle\bm{\psi}_{i}(\bm{w}), \bm{\varphi}_{i}(\bm{z})\rangle$ for some maps $\bm{\psi}_i, \bm{\varphi}_i: \mathbb{R}^{d_i}\to\mathbb{R}^{p_i}$ with $r_i := \dim^*(\sigma, Z_i), i \in [s]$. Denote 
	\[ \bm{\psi}(\widetilde{W}) = \left(\begin{smallmatrix}
	\bm{\psi}_{1}(W_{1}) & \cdots & \bm{0}  \\
	\vdots & \ddots & \vdots \\
	\bm{0} & \cdots & \bm{\psi}_{s}(W_{s}) \\
	\end{smallmatrix}\right), \text{ where } \bm{\psi}_{i}(W_{i}) \in \mathbb{R}^{p_i \times r_i}, \; \forall i \in [s]. \]
	Then the output of network becomes $ U \bm{\psi}(\widetilde{W})
	\left(\begin{smallmatrix}
	\bm{\varphi}_{1}(Z_1) \\
	\vdots \\
	\bm{\varphi}_{s}(Z_s)  \\
	\end{smallmatrix}\right)$. 
	Once $\bm{\psi}(\widetilde{W})$ has full column rank, we only need to change $U$ to obtain a non-increasing path to the zero loss by Lemma \ref{lemma:convex}. Otherwise, we can reach a full column rank $\bm{\psi}(\widetilde{W})$ by adjusting $U_i$s and $W_i$s following the proof of \citet[Corollary 10]{venturi2019spurious} or Lemma \ref{lemma:zero-path}. 
\end{proof}

\subsection{Proof of Theorem \ref{thm:non-linear}} 
% We need a key property of real analytic functions below:
\begin{lemma}[\citet{dang2015complex, mityagin2015zero}] \label{lemma:real-ana}
	If $f:\mathbb{R}^n\to \mathbb{R}$ is a real analytic function which is not identically zero,
	then $ \{x \in \mathbb{R}^n| f(x) = 0\} $ has Lebesgue measure zero.
\end{lemma}

Now we turn to the proof of Theorem \ref{thm:non-linear}.

\begin{proof}
% 	If $\sigma(\widetilde{W}X)$ has full column rank, we only need to change the final layer $U$ to approach zero loss because $p\geq n$. Otherwise, 
	Suppose $\mathrm{rank}(\sigma(\widetilde{W}X))=t-1<n$, and without loss of generality, assume that the first $t-1$ rows of $\sigma(\widetilde{W}X) \in \mathbb{R}^{p \times n}$ are linearly independent. 
	Then from Lemma \ref{lemma:zero-path}, we can construct a loss-invariant linear path from $(U,\widetilde{W})$ to $(U', \widetilde{W})$ where $U'$ has $p-t+1$ zero columns. We only need to show that $\sigma(\widetilde{W}X)$ can become full column rank after replacing the last $p-t+1$ rows of $\widetilde{W}$ while the loss is unchanged due to zero columns in $U'$.
	
	First, we show that there exists $\bm{w}_t$ such that the first $t$ rows of  $\sigma(\widetilde{W}X)$ are linearly independent, where $\bm{w}_t$ is the $t$-th row of $\widetilde{W}$.
	We consider the determinant of the first $t$ rows and columns of $\sigma(\widetilde{W}X)$ as
	\[ g(\bm{w}_t):=\det \left(\begin{smallmatrix}
	\sigma(\bm{w}_1^\top (Z_1)_{\cdot, 1}) & \dots & \sigma(\bm{w}_1^\top (Z_1)_{\cdot, i}) & \dots & \sigma(\bm{w}_1^\top (Z_1)_{\cdot, t}) \\
	\vdots & \ddots & \vdots & \ddots &	\vdots \\
	\sigma(\bm{w}_{t}^\top (Z_t)_{\cdot, 1}) & \dots & \sigma(\bm{w}_{t}^\top (Z_t)_{\cdot, i}) & \dots & \sigma(\bm{w}_{t}^\top (Z_t)_{\cdot, t}) \\
	\end{smallmatrix}\right). \]
	
	Then $g(\bm{w}_t)$ is real analytic function from Assumption \ref{ass:ac}. 
	If $g(\cdot)\not\equiv 0 $, we could find some $\bm{w}_t$ with $g(\bm{w}_t)\neq 0$ based on Lemma \ref{lemma:real-ana}, which already shows that the first $t$ rows of $\sigma(\widetilde{W}X)$ are linearly independent.
	Otherwise, $g(\cdot) \equiv 0 $, then we have
	\begin{align*}
		\frac{\partial^{k} g(\bm{w}_t)}{\partial (\bm{w}_t)_1^k} = \det \left(\begin{smallmatrix}
		\sigma(\bm{w}_1^\top (Z_1)_{\cdot, 1}) & \dots & \sigma(\bm{w}_1^\top (Z_1)_{\cdot, i}) & \dots & \sigma(\bm{w}_1^\top (Z_1)_{\cdot, t}) \\
		\vdots & \ddots & \vdots & \ddots & \vdots \\
		\sigma(\bm{w}_{t-1}^\top (Z_{t-1})_{\cdot, 1}) & \dots & \sigma(\bm{w}_{t-1}^\top (Z_{t-1})_{\cdot, i}) & \dots & \sigma(\bm{w}_{t-1}^\top (Z_{t-1})_{\cdot, t}) \\
		(Z_t)_{1, 1}^k\sigma^{(k)}(\bm{w}_t^\top (Z_t)_{\cdot, 1}) & \dots & (Z_t)_{1, i}^k\sigma^{(k)}(\bm{w}_t^\top (Z_t)_{\cdot, i}) & \dots & (Z_t)_{1, t}^k\sigma^{(k)}(\bm{w}_t^\top (Z_t)_{\cdot, t}) \\
		\end{smallmatrix}\right) \equiv 0.
	\end{align*}
	
	Since the first $t{-}1$ rows of $\sigma(\widetilde{W}X)$ are linearly independent, we obtain that the last row of $\frac{\partial^{k} g(\bm{w}_t)}{\partial (\bm{w}_t)_1^k}$ can be expressed by a linear combination of the remaining $t{-}1$ rows. 
	Thus, the last rows of $\frac{\partial^{l_i} g(\bm{w}_t)}{\partial (\bm{w}_t)_1^{l_i}}, i \in [t]$ are linearly dependent, showing that for all
	$ \bm{w}_t$,
	\begin{equation}\label{eq:det}
	\begin{aligned}
	\left(\begin{smallmatrix}
	(Z_t)_{1, 1}^{l_1}\sigma^{(l_1)}(\bm{w}_t^\top (Z_t)_{\cdot, 1}) &  \dots & (Z_t)_{1, t}^{l_1}\sigma^{(l_1)}(\bm{w}_t^\top (Z_t)_{\cdot, t})\\
	\vdots & \ddots & \vdots \\
	(Z_t)_{1, 1}^{l_t}\sigma^{(l_t)}(\bm{w}_t^\top (Z_t)_{\cdot, 1}) & \dots & (Z_t)_{1, t}^{l_t}\sigma^{(l_t)}(\bm{w}_t^\top (Z_t)_{\cdot, t}) \\
	\end{smallmatrix}\right) = 0.
	\end{aligned}
	\end{equation}
	Since $l_1, \dots, l_n \in \mathbb{N} $ form an arithmetic sequence, we denote $d= l_1-l_0\in\mathbb{N}_+$ and choose $\bm{w}_t=\bm{0}$, then the left side of Eq. (\ref{eq:det}) equals to 
	\[ \prod_{i=1}^{t}\sigma^{(l_i)}(0) \cdot \prod_{i=1}^{t} \left(Z_t\right)_{1, i}^{l_1} \cdot \prod_{1 \leq i < j \leq t}\left[(Z_t)_{1, j}^d-(Z_t)_{1, i}^d\right], \]
	where we use the property of Vandermonde matrix. However, Assumption \ref{ass:ac} implies that $\sigma^{(l_1)}(0), \dots,\sigma^{(l_t)}(0) \neq 0$. Furthermore, Assumption \ref{ass:data} implies $(Z_t)_{1, i}^d\neq  (Z_t)_{1, j}^d$ and $ (Z_t)_{1, i}^{l_1}\neq 0$ since $\left(Z_t\right)_{1, \cdot}$ includes all different non-zero feature (instead of $\bm{1}^\top$). 
	Therefore, we obtain contradiction, and we could find $\bm{w}_t$ to make the first $t$ rows of $\sigma(\widetilde{W}X)$ are linearly independent,  
	
	Second, we could repeat the above step until $\sigma(\widetilde{W}X)$ has full column rank. Moreover, previous argument also shows that $\mathrm{rank}(\sigma(\widetilde{W}X))=n, a.s.$ Particularly, any $n$ rows of $\sigma(\widetilde{W}X)$ have full rank with measure one.
	
	Now from $U'$ has $p - t + 1$ zero columns, we can adjust the $i$-th row of $\widetilde{W}$ with $i\geq t$ arbitrarily. Hence, we could reach some $\widetilde{W}'$ with the full column rank $\sigma(\widetilde{W}'X)$ according to the above argument. Then we can approach zero loss by a linear path from $U'$ to $U''=Y[\sigma(\widetilde{W}'X)]^+$ with non-increasing loss from Lemma \ref{lemma:convex}.
\end{proof}	

\subsection{Proof of Theorem \ref{thm:deep-non-linear}}
\begin{lemma}[Lemma 3 of \citet{li2018benefit}] \label{lemma:dawei}
	Suppose $\sigma(\cdot)$ is a non-constant analytic function.
	Then given $\bm{a}, \bm{b} \in \mathbb{R}^d$, the set $\Omega= \{\bm{z} \in \mathbb{R}^d | \sigma(\bm{a}^\top \bm{z}) = \sigma(\bm{b}^\top\bm{z}), \bm{a}\neq \bm{b}\}$ has zero measure.
\end{lemma}
Now we turn to the proof of Theorem \ref{thm:deep-non-linear}.

\begin{proof}
	We first consider the sparse network with a dense final layer. When applied to deep sparse networks, we need to show that all hidden layer outputs still satisfy Assumption \ref{ass:data}. 
	Denote certain hidden-layer outputs for each training sample as $\bm{z}_1, \dots, \bm{z}_n$, and $\mathcal{S}_{i j} :=\{\bm{w} | \sigma(\bm{w}^\top \bm{z}_i)^2 = \sigma(\bm{w}^\top \bm{z}_j)^2, \bm{z}_i-\bm{z}_j\neq \bm{0}\}$. 
	From Lemma \ref{lemma:dawei} and note that $\sigma^2(\cdot)$ is still real analytic, then $\mathcal{S}_{i j}$ has zero measure and $\cup_{i\neq j}\mathcal{S}_{i j}$ has zero measure.
	Moreover, from Lemma \ref{lemma:real-ana}, $\cup_{i=1}^n \{\bm{w} | \sigma(\bm{w}^\top \bm{z}_i) = 0, \bm{z}_i\neq \bm{0}\}$ has zero measure.
	Therefore, each hidden-layer output satisfies 1) in Assumption \ref{ass:data} with measure one. Combining with the proof of Theorem \ref{thm:non-linear}, we obtain that each hidden-layer output has full rank with measure one by choosing a square submatrix.
	Therefore, the full rank argument can be applied to the next layer until the last hidden layer. 

	If there exists a spurious valley $\mathcal{T}$ with a interior point $\bm{\theta}$, then from the measure one argument, we can find a point $\bm{\theta}'\in \mathcal{T}$ which is arbitrary close to $\bm{\theta}$ and the last hidden-layer output has full rank.
	Using over-parametrized regime in the last layer, we derive a non-increasing path from $\bm{\theta}'$ to approach zero loss by only adjusting the final layer weight.
	Thus we obtain the contradiction for this spurious valley.
	
	Now if the final layer weight $\widetilde{U}$ is a sparse matrix, we can also rearrange the order to formulate the output as
	\[ \widetilde{U} \sigma\left(\begin{smallmatrix}
	W_1 Z_1 \\ \vdots \\ W_s Z_s \\
	\end{smallmatrix}\right) = \left(\begin{smallmatrix}
	U_1 &  \dots & \bm{0} \\
	\vdots & \ddots & \vdots \\
	\bm{0} &  \dots & U_r    \\
	\end{smallmatrix}\right)\left(\begin{smallmatrix}
	G_1 \\ \vdots \\ G_r \\
	\end{smallmatrix}\right), \]
	where $G_i$s are some rows of the last hidden layer output matrix, and $U_i$s are dense matrices with no less than $n$ columns. 
	Repeating the argument of Theorem \ref{thm:non-linear}, $G_i$s have rank $n$ with measure one. Finally, the loss function is convex for each $U_i$. Thus, we can still construct a non-increasing path to zero loss by Lemma \ref{lemma:convex}.
\end{proof} 

% \section*{References}
\bibliography{reference}
\bibliographystyle{plainnat}

\end{document}